\documentclass{article}

\usepackage[nonatbib,final]{neurips_2020}

\usepackage{subcaption}
\usepackage{hyperref}
\usepackage{AVTcitations}
\usepackage{AVTcommands}
\usepackage{AVTenvironments}
\usepackage{microtype}

\usepackage{amsthm}
\newtheorem{theorem}{Theorem}
\newtheorem{proposition}[theorem]{Proposition}

\newtheorem{example}[theorem]{Example}
\newtheorem{definition}[theorem]{Definition}

\usepackage{nicefrac}

\usepackage{thmtools}
\usepackage{thm-restate}

\DeclareMathOperator{\hyp}{hyp}
\DeclarePairedDelimiterX\Set[2]{\lbrace}{\rbrace}%
 { #1 \,\delimsize|\, #2 }

\bibliography{references}

\usepackage{tikz}

\let\UrlFont\scshape
\makeatletter\def\@fnsymbol#1{\ensuremath{\ifcase#1\or *\or 1,4\or *2\or *1\or 3\or 1\or 2\or 4\fi}}\makeatother

\begin{document}

\title{Mat\'{e}rn Gaussian processes on Riemannian manifolds}

\author{\!\!Viacheslav Borovitskiy\thanks{Equal contribution. Correspondence to: \email{viacheslav.borovitskiy@gmail.com},\newline\email{a.terenin17@imperial.ac.uk}, and \email{pmostowsky@gmail.com}. \newline\hspace*{1em} Code available at \url{https://github.com/spbu-math-cs/Riemannian-Gaussian-Processes} and\newline \url{https://github.com/aterenin/SparseGaussianProcesses.jl}.}\,\,\footnotemark[2]
\quad Alexander Terenin\footnotemark[3]
\quad Peter Mostowsky\footnotemark[4]
\quad Marc Peter Deisenroth\footnotemark[5]\!
\\[2mm]
\footnotemark[6]\,\,\,St. Petersburg State University
\quad\,\,\,
\footnotemark[7]\,\,\,Imperial College London
\quad\,\,\,
\footnotemark[5]\,\,\,University College London
\\
\footnotemark[8]\,\,\,St. Petersburg Department of Steklov Mathematical Institute of Russian Academy of Sciences
}

\maketitle

\begin{abstract}
Gaussian processes are an effective model class for learning unknown functions, particularly in settings where accurately representing predictive uncertainty is of key importance. 
Motivated by applications in the physical sciences, the widely-used Mat\'{e}rn class of Gaussian processes has recently been generalized to model functions whose domains are Riemannian manifolds, by re-expressing said processes as solutions of stochastic partial differential equations.
In this work, we propose techniques for computing the kernels of these processes on compact Riemannian manifolds via spectral theory of the Laplace--Beltrami operator in a fully constructive manner, thereby allowing them to be trained via standard scalable techniques such as inducing point methods.
We also extend the generalization from the Mat\'{e}rn to the widely-used squared exponential Gaussian process. 
By allowing Riemannian Mat\'{e}rn Gaussian processes to be trained using well-understood techniques, our work enables their use in mini-batch, online, and non-conjugate settings, and makes them more accessible to machine learning practitioners.
\end{abstract}

\section{Introduction}

Gaussian processes (GPs) are a widely-used class of models for learning an unknown function within a Bayesian framework.
They are particularly attractive for use within decision-making systems, e.g. in Bayesian optimization \cite{snoek2012} and reinforcement learning \cite{deisenroth11, deisenroth13}, where well-calibrated uncertainty is crucial for enabling the system to balance trade-offs, such as exploration and exploitation.

A GP is specified through its mean and covariance kernel.
The Mat\'{e}rn family is a widely-used class of kernels, often favored in Bayesian optimization due to its ability to specify smoothness of the GP by controlling differentiability of its sample paths.
Throughout this work, we view the widely-used squared exponential kernel as a Mat\'{e}rn kernel with infinite smoothness.

Motivated by applications areas such as robotics \cite{jaquier19,calinon20} and climate science \cite{camps2016}, recent work has sought to generalize a number of machine learning algorithms from the vector space to the manifold setting.
This allows one to work with data that lives on spheres, cylinders, and tori, for example.
To define such a GP, one needs to define a positive semi-definite kernel on those spaces.

In the Riemannian setting, as a simple candidate generalization for the Mat\'{e}rn or squared exponential kernel, one can consider replacing Euclidean distance in the formula with the Riemannian geodesic distance.
Unfortunately, this approach leads to ill-defined kernels in many cases of interest \cite{feragen15}.

An alternative approach was recently proposed by \textcite{lindgren11}, who adopt a perspective introduced in the pioneering work of \textcite{whittle63} and define a Mat\'{e}rn GP to be the solution of a certain stochastic partial differential equation (SPDE) driven by white noise.
This approach generalizes naturally to the Riemannian setting, but is cumbersome to work with in practice because it entails solving the SPDE numerically.
In particular, setting up an accurate finite element solver can become an involved process, especially for certain smoothness values \cite{bolin17, bolin19}.
This also prevents one from easily incorporating recent advances in scalable GPs, such as sparse inducing point methods \cite{titsias09a, hensman13}, into the framework.
This in turn impedes one from easily employing mini-batch training, online training, non-Gaussian likelihoods, or incorporating GPs as differentiable components within larger models.

In this work, we extend Mat\'{e}rn GPs to the Riemannian setting in a fully constructive manner, by introducing Riemannian analogues of the standard technical tools one uses when working with GPs in Euclidean spaces.
To achieve this, we first study the special case of the $d$-dimensional torus $\bb{T}^d$.
Using ideas from abstract harmonic analysis, we view GPs on the torus as periodic GPs on $\R^d$, and derive expressions for the kernel and spectral measure of a Mat\'{e}rn GP in this case.

Building on this intuition, we generalize the preceding ideas to general compact Riemannian manifolds without boundary. 
Using insights from harmonic analysis induced by the Laplace--Beltrami operator, we develop techniques for computing the kernel and generalized spectral measure of a Mat\'{e}rn GP in this setting.
These expressions enable computations via standard GP approaches, such as Fourier feature or sparse variational methods, thereby allowing practitioners to easily deploy familiar techniques in the Riemannian setting.
We conclude by showcasing how to employ the proposed techniques through a set of examples.

\section{Gaussian processes}

Let $X$ be a set, and let $f: X \-> \R$ be a random function.
We say that $f \~[GP](\mu, k)$ if, for any~$n$ and any finite set of points $\v{x} \in  X^n$, the random vector $\v{f} = f(\v{x})$ is multivariate Gaussian with prior mean vector $\v\mu = \mu(\v{x})$ and covariance matrix $\m{K}_{\v{x}\v{x}} = k(\v{x},\v{x})$.
We henceforth, without loss of generality, set the mean function to be zero.

Given a set of training observations $(x_i, y_i)$, we let $y_i  = f(x_i) + \eps_i$ with $\eps_i \~[N](0,\sigma^2)$. 
Under the prior $f \~[GP](0,k)$ the posterior distribution $f\given\v{y}$ is another GP, with mean and covariance
\<
\E(f\given\v{y}) &= \m{K}_{(\cdot)\v{x}} (\m{K}_{\v{x}\v{x}} + \sigma^2\m{I})^{-1}\v{y}
&
\Cov(f\given\v{y}) &= \m{K}_{(\cdot,\cdot)} - \m{K}_{(\cdot)\v{x}} (\m{K}_{\v{x}\v{x}} + \sigma^2\m{I})^{-1} \m{K}_{\v{x}(\cdot)}
\> 
where $(\cdot)$ denotes an arbitrary set of test locations. The posterior can also be written
\[
\label{eqn:pathwise}
(f\given\v{y})(\cdot) = f(\cdot) + \m{K}_{(\cdot)\v{x}} (\m{K}_{\v{x}\v{x}} + \sigma^2\m{I})^{-1} (\v{y} - f({\v{x}}) - \v\eps)
\]
where equality holds in distribution \cite{wilson20}.
This expression allows one to sample from the posterior by first sampling from the prior, and transforming the resulting draws into posterior samples.

On $X = \R^d$, one popular choice of kernel is the \emph{Mat\'{e}rn} family with parameters $\sigma^2,\kappa,\nu$, defined as
\[
\label{eqn:matern-kernel-eucl}
k_{\nu}(x,x') = \sigma^2 \frac{2^{1-\nu}}{\Gamma(\nu)} \del{\sqrt{2\nu} \frac{\norm{x-x'}}{\kappa}}^\nu K_\nu \del{\sqrt{2\nu} \frac{\norm{x-x'}}{\kappa}}
 \]
where $K_\nu$ is the modified Bessel function of the second kind \cite{gradshteyn14}.
The parameters of this kernel have a natural interpretation: $\sigma^2$ directly controls variability of the GP, $\kappa$ directly controls the degree of dependence between nearby data points, and $\nu$ directly controls mean-square differentiability of the GP \cite{rasmussen06}.
As $\nu\->\infty$, the Mat\'{e}rn kernel converges to the widely-used squared exponential kernel 
\[
\label{eqn:rbf-kernel-eucl}
k_{\infty}(x,x') = \sigma^2 \exp\del[3]{-\frac{\norm{x-x'}^2}{2 \kappa^2}}
\]
which induces an infinitely mean-square differentiable GP.

For a bivariate function $k:X\x X\->\R$ to be a kernel, it must be \emph{positive semi-definite}, in the sense that for any $n$ and any $\v{x}\in X^n$, the kernel matrix $\m{K}_{\v{x}\v{x}}$ is positive semi-definite.
For $X = \R^d$, a translation-invariant kernel $k(x,x') = k(x - x')$ is called \emph{stationary}, and can be characterized via Bochner's Theorem. 
This result states that a translation-invariant bivariate function is positive definite if and only if it is the Fourier transform of a finite non-negative measure $\rho$, termed the \emph{spectral measure}.
This measure is an important technical tool for constructing kernels \cite{rasmussen06}, and for practical approximations such as \emph{Fourier feature} basis expansions \cite{rahimi08,hensman17}.

\subsection{A no-go theorem for kernels on manifolds}

We are interested in generalizing the Mat\'{e}rn family from the vector space setting to a compact Riemannian manifold $(M,g)$
such as the sphere or torus.
One might hope to achieve this by replacing Euclidean norms with the geodesic distances in \eqref{eqn:matern-kernel-eucl} and \eqref{eqn:rbf-kernel-eucl}.
In the latter case, this amounts to defining
\[
\label{eqn:rbf-geodesic}
k(x,x') = \sigma^2 \exp\del{-\frac{d_g(x,x')^2}{2 \kappa^2}}
\]
where $d_g$ is the geodesic distance with respect to $g$ on $M$.
Unfortunately, one can prove this is not generally a well-defined kernel.

\begin{theorem}
\label{thm:no-go}
Let $(M,g)$ be a complete, smooth Riemannian manifold without boundary, with associated geodesic distance $d_g$. If the geodesic squared exponential kernel \eqref{eqn:rbf-geodesic} is positive semi-definite for all $\kappa > 0$, then $M$ is isometric to a Euclidean space.
\end{theorem}
\begin{proof}
\textcite[Theorem 2]{feragen15}.
\end{proof}

Since Euclidean space is not compact, this immediately implies that \eqref{eqn:rbf-geodesic} is not a well-defined kernel on any compact Riemannian manifold without boundary.
We therefore call \eqref{eqn:rbf-geodesic} and its finite-smoothness analogues the \emph{na\"{i}ve generalization}.

In spite of this issue, the na\"{i}ve generalization is usually still positive semi-definite for some $\kappa$, and it has been used in a number of applied areas \cite{jaquier19}. 
\textcite{feragen16} proposed a number of open problems arising from these issues.
In Section \ref{sec:torus}, we show that, on the torus, the na\"{i}ve generalization is \emph{locally correct} in a sense made precise in the sequel.
We now turn to an alternative approach, which gives well-defined kernels in the general case.

\subsection{Stochastic partial differential equations} \label{sec:spde}

\begin{figure}
\begin{subfigure}{.33\textwidth}
  \centering
  \includegraphics[height=2.5cm]{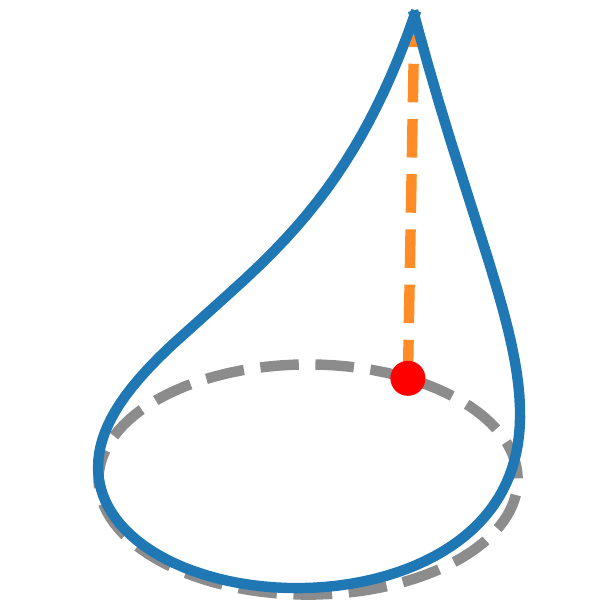}
  \label{fig:circle-kernel}
\end{subfigure}
\begin{subfigure}{.33\textwidth}
  \centering
  \includegraphics[height=2.5cm, trim=100 75 100 75, clip]{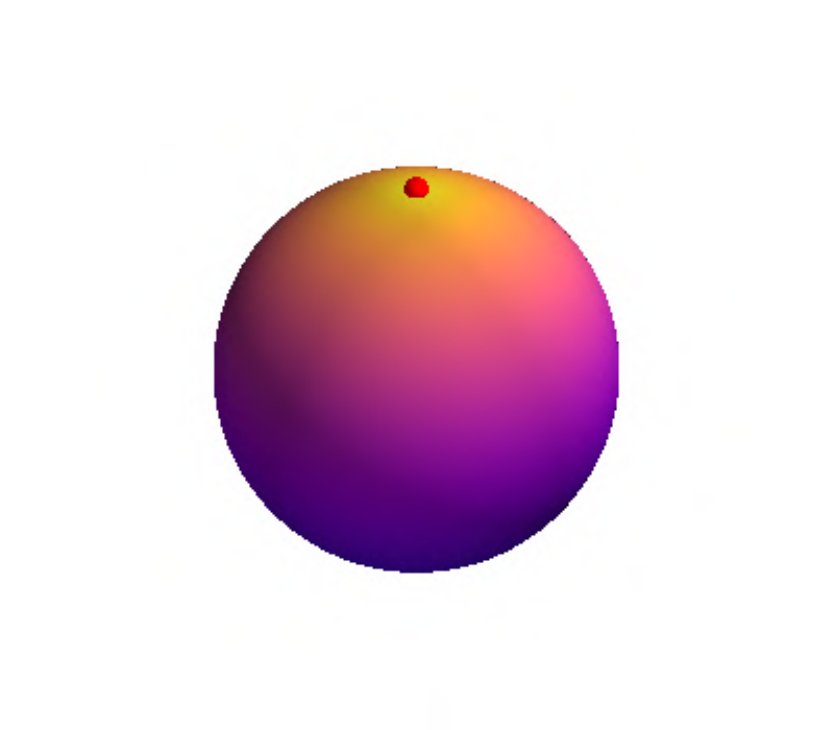}
  \label{fig:sphere-kernel}
\end{subfigure}
\begin{subfigure}{.33\textwidth}
  \centering
  \includegraphics[height=2.5cm, trim=0 30 0 30, clip]{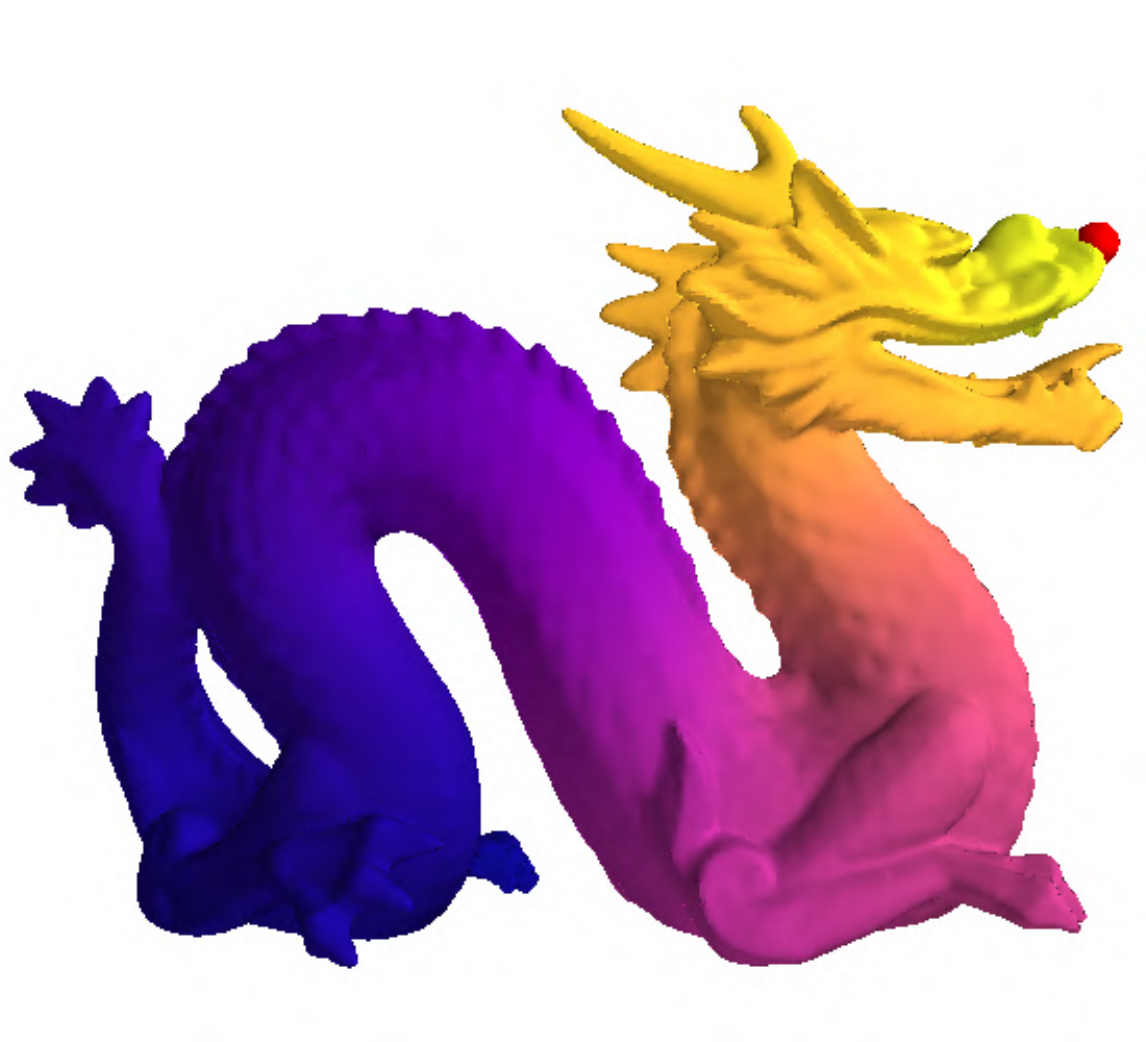}
  \label{fig:dragon-kernel}
\end{subfigure}
\caption{The Mat\'{e}rn kernel $k_{1/2}(x, \cdot)$, defined on a circle, sphere and dragon. The point $x$ is marked with a red dot.
The height of the solid line and color, respectively, give the value of the kernel.
}
\label{fig:kernels}
\end{figure}

\textcite{whittle63} has shown that Mat\'{e}rn GPs on $X = \R^d$ satisfy the stochastic partial differential equation
\[
\label{eqn:spde-matern}
\del{\frac{2 \nu}{\kappa^2} - \lap}^{\frac{\nu}{2} + \frac{d}{4}}f = \c{W}
\]
for $\nu < \infty$, where $\lap$ is the Laplacian and $\c{W}$ is Gaussian white noise re-normalized by a certain constant.
One can show using the same argument that the limiting squared exponential GP satisfies
\[
\label{eqn:spde-rbf}
e^{-\frac{\kappa^2}{4} \Delta} f = \c{W}
\]
where $e^{-\frac{\kappa^2}{4} \Delta}$ is the (rescaled) heat semigroup \cite{evans10, grigoryan2009}.
This viewpoint on GPs has recently been reintroduced in the statistics literature by \textcite{lindgren11}, and a number of authors, including \textcite{simpson12, sarkka13b}, have used it to develop computational techniques, notably in the popular \textsc{inla} package \cite{rue09}.

One advantage of the SPDE definition is that generalizing it to the Riemannian setting is straightforward: one simply replaces $\lap$ with the Beltrami Laplacian and $\c{W}$ with the canonical white noise process with respect to the Riemannian volume measure.
The kernels of these GPs, computed in the sequel, are illustrated in Figure \ref{fig:kernels}.
Unfortunately, the SPDE definition is somewhat non-constructive: it is not immediately clear how to compute the kernel, and even less clear how to generalize familiar tools to this setting.
In practice, this restricts one to working with PDE-theoretic discretization techniques, such as Galerkin finite element methods, the efficiency of which depend heavily on the smoothness of $f$, and which can require significant hand-tuning to ensure accuracy.
It also precludes one from working in non-conjugate settings, such as classification, or from using recently-proposed techniques for scalable GPs via sparse inducing point methods \cite{titsias09a, hensman13, hensman17}, as they require one to either be able to compute the kernel point-wise, or compute the spectral measure, or both.

\subsection{State of affairs and contribution}

In this work, our aim is to generalize the standard theoretical tools available for GPs on $\R^d$ to the Riemannian setting.
Our strategy is to first study the problem for the special case of a $d$-dimensional torus.
Here, we provide expressions for the kernel of a Mat\'{e}rn GP in the sense of \textcite{whittle63} via \emph{periodic summation}, which yields a series whose first term is the na\"{i}ve generalization.
Building on this intuition, we develop a framework using Laplace--Beltrami eigenfunctions that allows us to provide expressions for the kernel and generalized spectral measure of a Mat\'{e}rn GP on a general compact Riemannian manifold without boundary.
The framework is fully constructive and compatible with sparse GP techniques for scalable training.

A number of closely related ideas, beyond those described in the preceding sections, have been considered in various literatures.
\textcite{solin2020} used ideas based on spectral theory of the Laplace--Beltrami operator to approximate stationary covariance functions on bounded domains of Euclidean spaces.
These ideas were applied, for instance, to model ambient magnetic fields using Gaussian processes by \textcite{solin2018}.
An analog of the expression we provide in equation \eqref{eqn:mani-matern-formula} for the Riemannian Mat\'{e}rn kernel was concurrently proposed as a practical GP model by \textcite{coveney2020}---in this work, we \emph{derive} said expression from the SPDE formulation of Mat\'{e}rn GPs.
Finally, the Riemannian squared exponential kernel, also sometimes called the heat or diffusion kernel, has been studied by \textcite{gao2019}.
We connect these ideas with stochastic partial differential equations.

In this work, we concentrate on Gaussian processes $f : M \-> \R$ whose \emph{domain} is a Riemannian manifold.
We do not study models $f : \R \-> M$ where the \emph{range} is a Riemannian manifold---this setting is explored by \textcite{mallasto2018}.

\section{A first example: the $d$-dimensional torus} \label{sec:torus}

\begin{figure}
\<
&\mathclap{
\begin{tikzpicture}[baseline={([yshift=-.5ex]current bounding box.center)}]
\node at (-0.05,-0.1) {};
\node at (0,1) {};
\draw[thick] (0,0) node[circle, fill, inner sep=1] {} to[out=0, in=270] (0.5,0.5) node[circle, fill, inner sep=1] {} ;
\end{tikzpicture}
}
&
&\mathclap{
\begin{tikzpicture}[baseline={([yshift=-.5ex]current bounding box.center)}]
\draw[thick] (0,0) node[circle, fill, inner sep=1] {} to[out=180, in=270] (-0.5,0.5) to[out=90, in=180] (0,1) to[out=0, in=90] (0.5, 0.5) node[circle, fill, inner sep=1] {};
\end{tikzpicture}
}
&
&\mathclap{
\begin{tikzpicture}[baseline={([yshift=-.5ex]current bounding box.center)}]
\draw[thick] (0,0) node[circle, fill, inner sep=1] {} to[out=0, in=270] (0.5,0.5) to[out=90, in=0] (0,1.05) to[out=180, in=90] (-0.55,0.5) to[out=270, in=180] (0, -0.1) to[out=0, in=270] (0.6, 0.5) node[circle, fill, inner sep=1] {};
\end{tikzpicture}
}
&
&\mathclap{
\begin{tikzpicture}[baseline={([yshift=-.5ex]current bounding box.center)}]
\draw[thick] (0,0) node[circle, fill, inner sep=1] {} to[out=180, in=270] (-0.5,0.5) to[out=90, in=180] (0,1) to[out=0, in=90] (0.5, 0.45) to[out=270,in=0] (0,-0.1) to[out=180, in=270] (-0.6,0.5) to[out=90, in=180] (0,1.1) to[out=0, in=90] (0.6,0.5) node[circle, fill, inner sep=1] {};
\end{tikzpicture}
}
&
&\mathclap{
\begin{tikzpicture}[baseline={([yshift=-.5ex]current bounding box.center)}]
\draw[thick] (0,0)
node[circle, fill, inner sep=1] {}
to[out=0, in=270] (0.5,0.5)
to[out=90, in=0] (0,1.05)
to[out=180, in=90] (-0.55,0.5)
to[out=270, in=180] (0, -0.1)
to[out=0, in=270] (0.6, 0.5)
to[out=90, in=0] (0,1.15)
to[out=180, in=90] (-0.65,0.5)
to[out=270, in=180] (0, -0.2)
to[out=0, in=270] (0.7, 0.5)
node[circle, fill, inner sep=1] {};
\end{tikzpicture}
}
\nonumber
\\
&\mathclap{\norm{x - x'}}
&
&\mathclap{\norm{x - x' - 1}}
&
&\mathclap{\norm{x - x' + 1}}
&
&\mathclap{\norm{x - x' - 2}}
&
&\mathclap{\norm{x - x' + 2}}
\nonumber
\>
\caption{
The distances being considered in definitions \eqref{eqn:matern-torus} and \eqref{eqn:rbf-torus}.
}
\label{fig:dist-torus}
\end{figure}
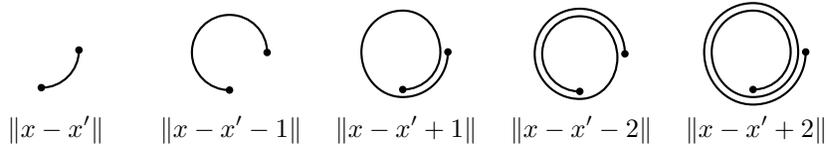

To begin our analysis and build intuition, we study the $d$-dimensional torus $\bb{T}^d$, which is defined as the product manifold $\bb{T}^d = \bb{S}^1 \x ... \x \bb{S}^1$ where $\bb{S}^1$ denotes a unit circle\footnote{Note that $\bb{T}^2 = \bb{S}^1 \times \bb{S}^1$ is diffeomorphic but \emph{not} isometric to the usual donut-shaped torus whose metric is induced by embedding in $\R^3$. This is important, because it is the Riemannian metric structure that gives rise to the Laplace--Beltrami operator and hence to the generalized Mat\'{e}rn and squared exponential kernels. Diffeomorphisms do not necessarily preserve metric structure, so they may not preserve kernels.}.
Since functions on a circle can be thought of as periodic functions on $\R$, and similarly for $\bb{T}^d$ and $\R^d$, defining a kernel on a torus is equivalent to defining a periodic kernel.
For a general function $f : \R^d \-> \R$, one can transform it into a function $g : \bb{T}^d  \-> \R$ by \emph{periodic summation}
\[ \label{eqn:periodic-summation}
g(x_1,...,x_d) = \sum_{n\in\Z^d} f(x_1 + n_1,...,x_d + n_d)
\]
where $x_j \in [0, 1)$ is identified with the angle $2\pi x_j$ and the point $\exp(2 \pi i x_j) \in \bb{S}^1$.
Define addition of two points in $\bb{S}^1$ by the addition of said numbers modulo $1$, and define addition in $\bb{T}^d$ component-wise.

Periodic summation preserves positive-definiteness, since it preserves positivity of the Fourier transform, which by Bochner's theorem is equivalent to positive-definiteness---see \textcite[Section 4.4.4]{scholkopf2002} for a formal proof.
This gives an easy way to construct positive-definite kernels on~$\bb{T}^d$.
In particular, we can generalize Mat\'{e}rn and squared exponential GPs from $\R^d$ to $\bb{T}^d$ by defining
\[
\label{eqn:matern-torus}
k_{\nu}(x,x')
=
\sum_{n\in\Z^d}
    \frac{\sigma^22^{1-\nu}}{C'_\nu\Gamma(\nu)}
    \del{
        \sqrt{2\nu}
        \frac{\norm{x-x' + n}}{\kappa}
    }^\nu
    K_\nu \del{
        \sqrt{2\nu}
        \frac{\norm{x-x' + n}}{\kappa}
    }
\]
where $C'_{(\cdot)}$ is a constant given in Appendix \ref{apdx:examples} to ensure $k_{(\cdot)}(x,x) = \sigma^2$, and
\[
\label{eqn:rbf-torus}
k_{\infty}(x,x')
=
\sum_{n\in\Z^d}
    \frac{\sigma^2}{C'_\infty}
    \exp\del{
        -
        \frac{\norm{x-x' + n}^2}{2 \kappa^2}
    }
\]
respectively.
We prove that these are the covariance kernels of the SPDEs introduced previously.

\begin{proposition}
\label{prop:matern-and-se-torus}
The Mat\'{e}rn (squared exponential) kernel $k$ in  \eqref{eqn:matern-torus}  (resp. \eqref{eqn:rbf-torus}) is the covariance kernel of the Mat\'{e}rn (resp. squared exponential) Gaussian process in the sense of \textcite{whittle63}.
\end{proposition}

\begin{proof}
Appendix \ref{apdx:torus-proof}.
\end{proof}

This result offers an intuitive explanation for \emph{why} the na\"{i}ve generalization based on the geodesic distance might fail to be positive semi-definite on non-Euclidean spaces for all length scales, yet work well for smaller length scales: on $\bb{T}^d$, it is \emph{locally correct} in the sense that it is equal to the first term in the periodic summation \eqref{eqn:matern-torus}.
To obtain the full generalization, one needs to take into account not just geodesic paths, but geodesic-like paths which include loops around the space---a Mat\'{e}rn GP incorporates global topological structure of its domain.
For the circle, these are visualized in Figure \ref{fig:dist-torus}.
For spaces where this structure is even more elaborate, definitions based purely on geodesic distances may not suffice to ensure positive semi-definiteness or good numerical behavior.
We conclude by presenting a number of practical formulas for Mat\'{e}rn kernels on the circle.

\begin{example}[Circle]
Take $M = \bb{S}^1$. For $\nu = \infty$, the kernel and spectral measure are
\<
k_{\infty}(x, x^\prime)
&=
\frac{\sigma^2}{C_\infty} \vartheta_3(\pi(x - x'), \exp(-2 \pi^2 \kappa^2))
&
\rho_{\infty}(n)
&=
\frac{\sigma^2}{C_\infty} \exp(-2 \pi^2 \kappa^2 n^2)
\>
where $n \in \Z$, $\vartheta_3(\cdot, \cdot)$ is the third Jacobi theta function \cite[equation 16.27.3]{abramowitz1972}, and $C_\infty = \vartheta_3(0, \exp(-2 \pi^2 \kappa^2))$.
This kernel is normalized to have variance $\sigma^2$.
\end{example}

\begin{example}[Circle]
Take $M = \bb{S}^1$. For $\nu = 1/2$, the kernel and spectral measure are
\<
k_{1/2}(x, x^\prime)
&=
\frac{\sigma^2}{C_{1/2}}
\cosh \del{\frac{\abs{x-x'} - 1/2}{\kappa}}
&
\rho_{1/2}(n)
&=
\frac{
    2\sigma^2
    \sinh\del{\nicefrac{1}{2 \kappa}}
    }
    {
    C_{1/2} \kappa
    }
\del{\frac{1}{\kappa^2} + 4\pi^2 n^2 }^{-1}
\!\!\!\!\!
\>
where $C_{1/2} = \cosh\del{\nicefrac{1}{2 \kappa} }$.
This kernel is normalized to have variance $\sigma^2$.
Note that $0 \!\leq\! \abs{x-x'} \!\leq\! \frac{1}{2}$.
\end{example}

A derivation and more general formula, valid for $\nu = 1/2 + n$, $n \in \N$, can be found in Appendix \ref{apdx:examples}.
Note that these spectral measures are \emph{discrete}, as the Laplace--Beltrami operator has discrete spectrum.
Finally, we give the Fourier feature approximation \cite{rahimi08,hensman17} of the GP prior on $\bb{T}^1 = \bb{S}^1$,~which~is
\<
f(x)
&\approx
\sum_{n=-N}^{N}
\sqrt{\rho_{\nu}(n)}
\del[1]{
    w_{n, 1} \cos(2 \pi n x)
    +
    w_{n, 2} \sin(2 \pi n x)
}
&
w_{n,j} &\~[N](0,1)
.
\>
We have defined Mat\'{e}rn and squared exponential GPs on $\bb{T}^d$ and given expressions for the kernel, spectral measure, and Fourier features on $\bb{T}^1$.
With sharpened intuition, we now study the general~case.

\section{Compact Riemannian manifolds} \label{sec:comp_man}

The arguments used in the preceding section are, at their core, based on ideas from abstract harmonic analysis connecting $\R^d$, $\bb{T}^d$, and $\Z^d$ as topological groups.
This connection relies on the algebraic structure of groups, which does not exist on a general Riemannian manifold.
As a result, different notions are needed to establish a suitable framework.

Let $(M,g)$ be a compact Riemannian manifold without boundary, and let $\lap_g$ be the Laplace--Beltrami operator.
Our aim is to compute the covariance kernel of the Gaussian processes solving the SPDEs \eqref{eqn:spde-matern} and \eqref{eqn:spde-rbf} in this setting.
Mathematically, this amounts to introducing an appropriate formalism so that one can calculate the desired expressions using spectral theory.
We do this in a fully rigorous manner in Appendix \ref{apdx:theory}, while here we present the main ideas and results.

First, we discuss how the operators on the left-hand side of SPDEs \eqref{eqn:spde-matern} and \eqref{eqn:spde-rbf} are defined.
By compactness of $M$, $-\lap_g$ admits a countable number of eigenvalues, which are non-negative and can be ordered to form a non-decreasing sequence with $\lambda_n\->\infty$ for $n\->\infty$.
Moveover, the corresponding eigenfunctions form an orthonormal basis $\{f_n\}_{n\in\Z_+}$ of $L^2(M)$, and $-\lap_g$ admits the representation
\[ \label{eqn:lap_spec}
-\lap_g f = \sum_{n=0}^\infty \lambda_n \innerprod{f}{f_n} f_n
\]
which is termed the \emph{Sturm--Liouville decomposition} \cite{chavel1984,canzani13}.
This allows one to define the operators $\Phi(-\lap_g)$ for a function $\Phi: [0, \infty) \to \R$, by replacing $\lambda_n$ with $\Phi(\lambda_n)$ in \eqref{eqn:lap_spec}, and specifying appropriate function spaces as domain and range to ensure convergence of the series in a suitable sense.
This idea is called \emph{functional calculus} for the operator $-\lap_g$.
Using it, we define
\< \label{eqn:operator_definitions}
\del{\frac{2 \nu}{\kappa^2} - \lap_g}^{\frac{\nu}{2}+\frac{d}{4}} f
&=
\sum_{n=0}^\infty \del{\frac{2 \nu}{\kappa^2} + \lambda_n}^{\frac{\nu}{2}+\frac{d}{4}} \innerprod{f}{f_n}   f_n
\\
e^{-\frac{\kappa^2}{4} \lap_g} f
&=
\sum_{n=0}^\infty e^{\frac{\kappa^2 \lambda_n}{4}} \innerprod{f}{f_n} f_n
.
\>
Figure \ref{fig:eigenfuntions} illustrates the eigenfunctions $f_n$.
Note that when $M = \bb{T}^d$, the orthonormal basis $\{f_n\}_{n\in\Z_+}$ consists of sines and cosines, and thus the corresponding functional calculus is defined in terms of standard Fourier series.
This also agrees with the usual way of defining such operators in the Euclidean case using the Fourier transform.

\begin{figure}
\centering
\includegraphics[height=1.875cm, trim=0 0 20 0, clip]{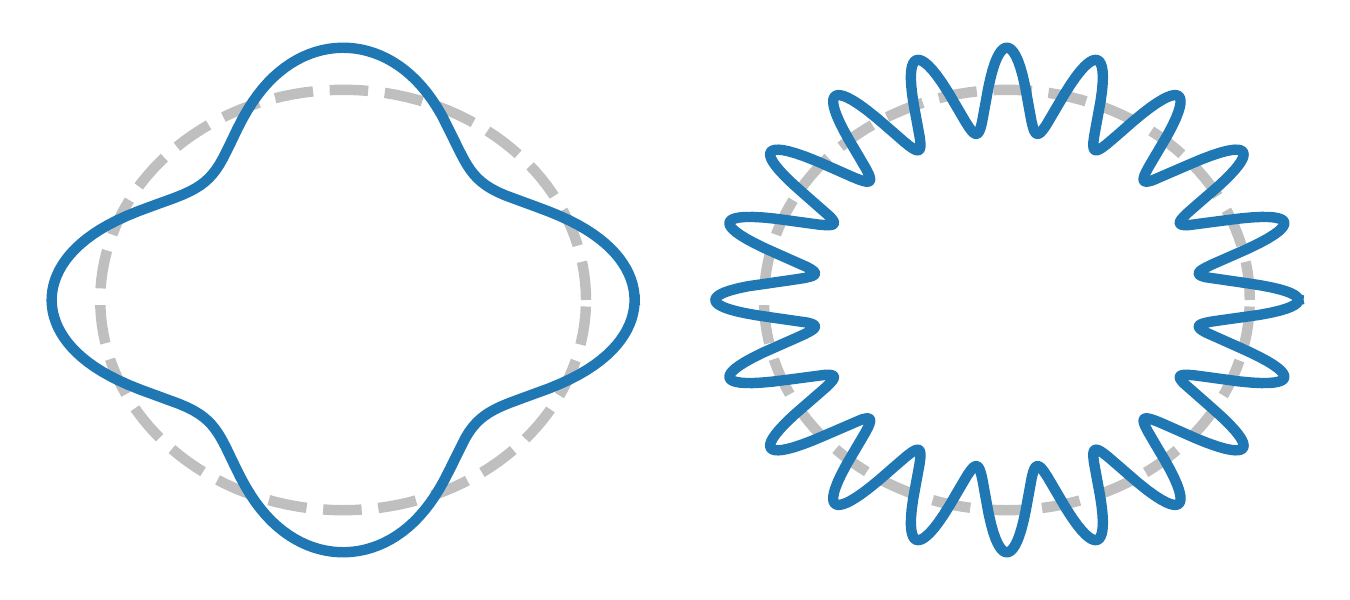}
\includegraphics[height=1.875cm, trim=10 10 0 10, clip]{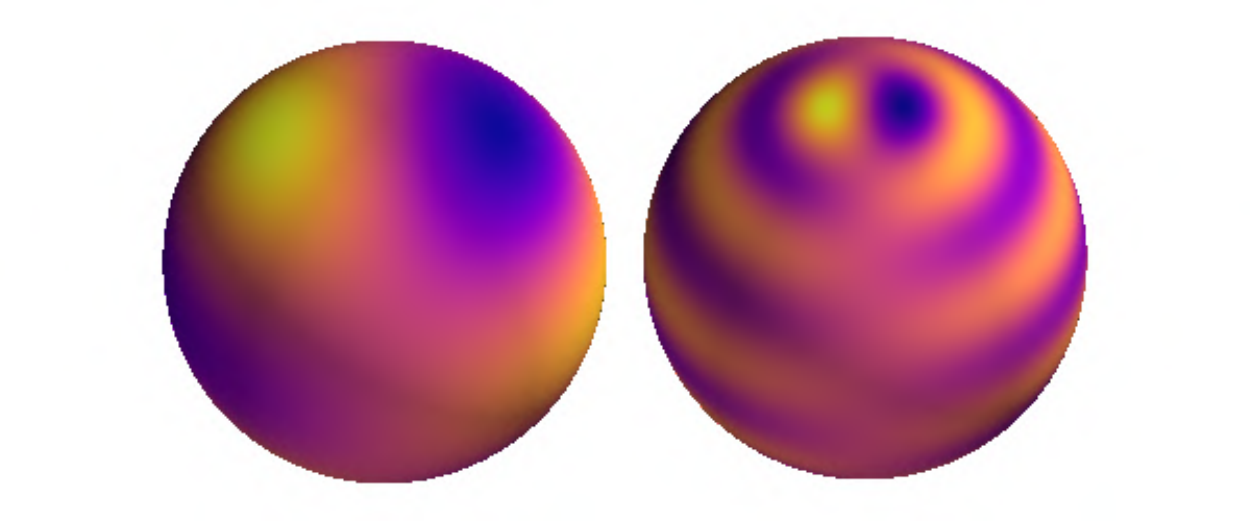}  
\includegraphics[height=1.875cm, trim=0 20 0 20, clip]{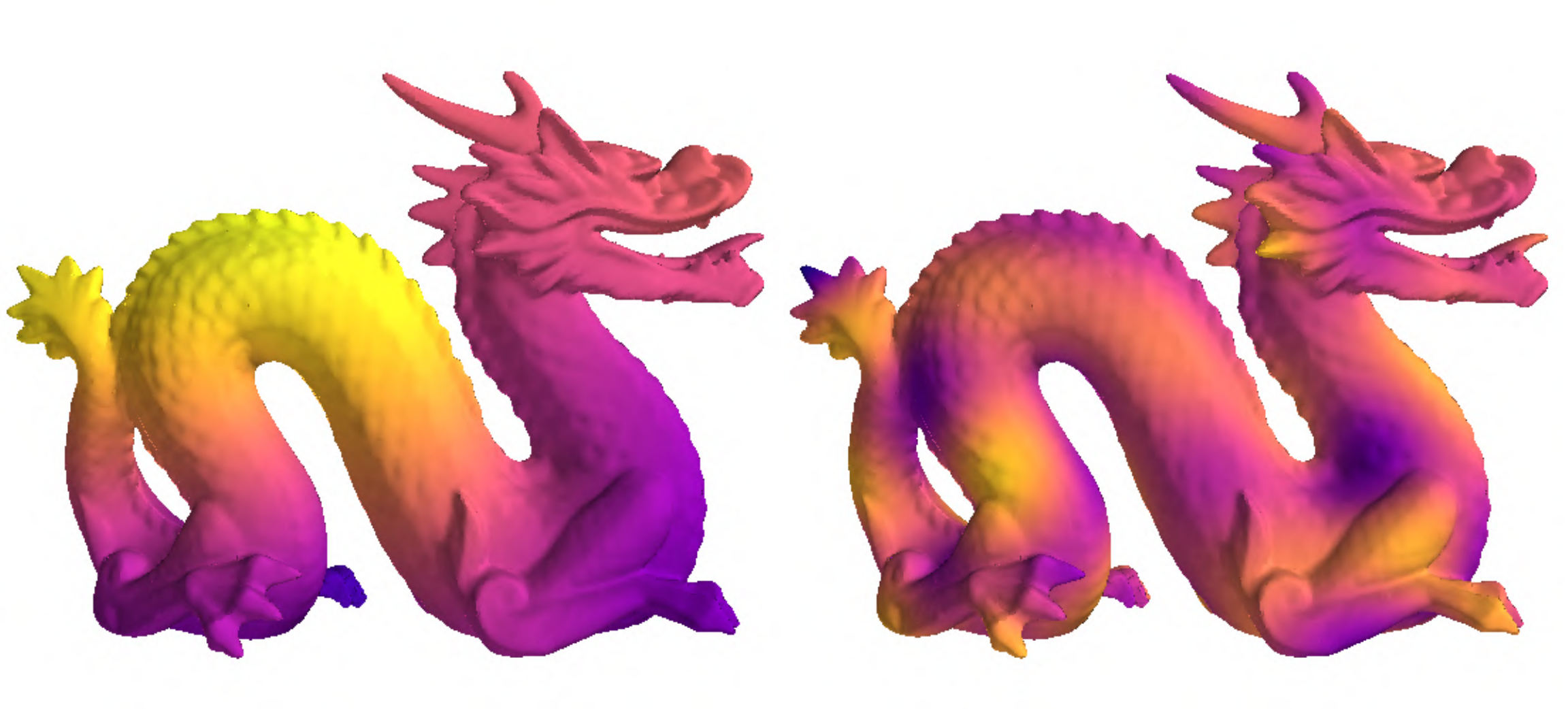} 
\caption{Examples of eigenfunctions of Laplace--Beltrami operator on a circle, sphere, and dragon. For the circle, the value of the eigenfunction is given by the (signed) distance between the solid line and dashed unit circle.
For the sphere and dragon, the value of the eigenfunction is given by the color.
}
\label{fig:eigenfuntions}
\end{figure}

Next, we proceed to define the remaining parts of the SPDEs.
The theory of stochastic elliptic equations described in \textcite{lototsky17} gives an appropriate notion of white noise $\c{W}$ for our setting, as well as a way to uniquely solve SPDEs of the form $\c{L} f = \c{W}$, where $\c{L}$ is a bounded linear bijection between a pair of Hilbert spaces.
We show that the operators
\<
\del{\frac{2 \nu}{\kappa^2} - \lap_g}^{\frac{\nu}{2}+\frac{d}{4}} &: H^{\nu + \frac{d}{2}}(M) \to L^2(M)
&
e^{\frac{\kappa^2}{4} \lap_g} &: \c{H}^{\frac{\kappa^2}{2}}(M) \to L^2(M)
\>
are bounded and invertible, where $H^s(M)$ are appropriately defined Sobolev spaces on the manifold, and $\c{H}^s(M)$ are the \emph{diffusion spaces} studied by \textcite{devito19}.

We prove that the solutions of our SPDEs in the sense of \textcite{lototsky17} are Gaussian processes with kernels equal to the reproducing kernels of the spaces $H^{\nu + d/2}(M)$ and $\c{H}^{\kappa^2/2}(M)$, which are given by \textcite{devito19}.
Summarizing, we get the following.

\begin{theorem}
Let $\lambda_n$ be eigenvalues of $-\lap_g$, and let $f_n$ be their respective eigenfunctions.
The kernels of the Mat\'{e}rn and squared exponential GPs on $M$ in the sense of \textcite{whittle63} are given by
\<
\label{eqn:mani-matern-formula}
k_{\nu}(x, x') &= \frac{\sigma^2}{C_\nu}\sum_{n=0}^\infty \del{\frac{2 \nu}{\kappa^2} + \lambda_n}^{-\nu-\frac{d}{2}} f_n(x)f_n(x')
\\
\label{eqn:mani-rbf-formula}
k_{\infty}(x, x') &= \frac{\sigma^2}{C_\infty}\sum_{n=0}^\infty e^{-\frac{\kappa^2}{2} \lambda_n} f_n(x)f_n(x')
\>
where $C_{(\cdot)}$ are normalizing constants chosen so that the average variance\footnote{The marginal variance $k_{(\cdot)}(x, x)$ can depend on $x$, thus we normalize the kernel by the average variance.}
over the manifold satisfies $\vol_g(M)^{-1} \int_X k_{(\cdot)}(x, x) \d x = \sigma^2$.
\end{theorem}

\begin{proof}
Appendix \ref{apdx:theory}.
\end{proof}

Our attention now turns to the spectral measure.
In the Euclidean case, the spectral measure, assuming sufficient regularity, is absolutely continuous---its Lebesgue density is given by the Fourier transform of the kernel. 
In the case of $\bb{T}^d$, the spectral measure is discrete---its density with respect to the counting measure is given by the Fourier coefficients of the kernel.
Like in the case of the torus, for a compact Riemannian manifold the spectral measure is discrete---its density with respect to the counting measure is given by the generalized Fourier coefficients of the kernel with respect to the orthonormal basis $f_n(x)f_{n'}(x')$ on $L^2(M\x M)$.
For Mat\'{e}rn and square exponential GPs, these are
\<
\label{eqn:matern-spectral-manifold}
\rho_{\nu}(n) &=  \frac{\sigma^2}{C_\nu} \del{\frac{2 \nu}{\kappa^2} + \lambda_n}^{-\nu-\frac{d}{2}}
&
\rho_{\infty}(n) &= \frac{\sigma^2}{C_\infty} \exp\del{-\frac{\kappa^2}{2} \lambda_n}
&
n &\in \N
.
\>
This allows one to recover most tools used in spectral theory of GPs.
In particular, one can construct a regular Fourier feature approximation of the GPs by taking the top-$N$ eigenvalues, and writing
\< \label{eqn:ff_approx}
f(x)
&\approx
\sum\limits_{n=0}^{N-1}
\sqrt{\rho(n)}
w_n
f_n(x)
&
w_n &\~[N](0,1)
.
\>
Other kinds of Fourier feature approximations, such as random Fourier features, are also possible.
We now illustrate an example in which these expressions simplify.
\begin{example}[Sphere]
Take $M = \bb{S}^d$ to be the $d$-dimensional sphere.
Then we have
\[
k_\nu(x,x') = \sum_{n=0}^\infty c_{n,d}\,\rho_\nu(n)\, \c{C}^{(d-1)/2}_n \del[2]{\cos\del[1]{d_g(x,x')}}
\]
where $c_{n,d}$ are constants given in Appendix \ref{apdx:examples}, $\c{C}_n^{(\cdot)}$ are the Gegenbauer polynomials, $d_g$ is the geodesic distance, and $\rho_\nu(n)$ can be expressed explicitly in terms of $\lambda_n = n(n+d-1)$ using \eqref{eqn:matern-spectral-manifold}. 
See Appendix \ref{apdx:examples} for details on the corresponding Fourier feature approximation.
\end{example}

A derivation with further details can be found in Appendix \ref{apdx:examples}.
Similar expressions are available for many other manifolds, where the Laplace--Beltrami eigenvalues and eigenfunctions are known.

\subsection{Summary}

We conclude by summarizing the presented method of computing the kernel of Riemannian Mat\'{e}rn Gaussian processes defined by SPDEs.
The key steps are as follows.

\1 Obtain the Laplace--Beltrami eigenpairs for the given manifold, either analytically or numerically.
This step needs to be performed once in advance.
\2 Approximate the kernel using a finite truncation of the infinite sums \eqref{eqn:mani-matern-formula} or \eqref{eqn:mani-rbf-formula}.
\0 

This kernel approximation can be evaluated pointwise at any locations, fits straightforwardly into modern automatic differentiation frameworks, and is simple to work with.
The resulting truncation error will depend on the smoothness parameter $\nu$, dimension $d$, and eigenvalue growth rate, which is quantified by Weyl's law \cite{zelditch2017}.
For $\nu < \infty$ convergence will be polynomial, and for $\nu = \infty$ it will be exponential.
If $\sigma^2$ is trainable, the constant $C_\nu$ which normalizes the kernel by its average variance can generally be disregarded.
If Fourier feature approximations of the prior are needed, for instance, to apply the pathwise sampling technique of \textcite{wilson20}, they are given by \eqref{eqn:ff_approx}.

\section{Illustrated Examples} \label{sec:examples}

Here we showcase two examples to illustrate the theory: dynamical system prediction and sample path visualization.
We focus on simplified settings to present ideas in an easy-to-understand manner.

\subsection{Dynamical system prediction}

\begin{figure}
\begin{center}
\begin{subfigure}{.26\linewidth}
\centering
\includegraphics[trim=50 45 50 50,clip]{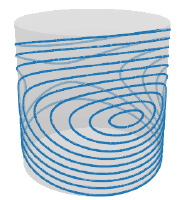}
\caption{Ground truth}
\label{fig:cylinder-ground-truth}
\end{subfigure}
\begin{subfigure}{.26\linewidth}
\centering
\includegraphics[trim=50 45 50 50,clip]{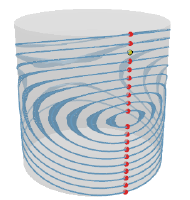}
\caption{Posterior 95\% intervals}
\label{fig:cylinder-learned-trajectories}
\end{subfigure}
\begin{subfigure}{.44\linewidth}
\centering
\includegraphics[scale=0.75,trim=0 4 0 12,clip]{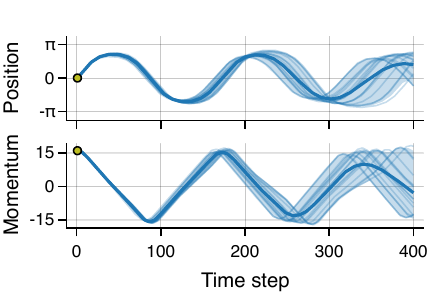}
\caption{Posterior samples for one trajectory}
\label{fig:cylinder-single-trajectory}
\end{subfigure}
\end{center}
\caption{Visualization of the dynamical system's learned phase diagram. 
Middle: we simulate 40 trajectories starting at the red dots, integrate the learned Hamilton's equations forward and backward in time until they approximately intersect other trajectories, and plot 95\% intervals in phase space.
Right: we simulate the trajectory beginning from the yellow dot, and plot mean and 95\% intervals.
}
\label{fig:cylinder}
\end{figure}

We illustrate how Riemannian squared exponential GPs can be used for predicting dynamical systems while respecting the underlying geometry of the configuration space the system is defined on.
This is an important task in robotics, where GPs are often trained within a model-based reinforcement learning framework \cite{deisenroth11,deisenroth13}. 
Here, we consider a purely supervised setup, mimicking the model learning inner loop of said framework.

For a prototype physical system, consider an ideal pendulum, whose configuration space is the circle $\bb{S}^1$, and whose phase space is the cotangent bundle $T^*\bb{S}^1$, which is isometric to the cylinder $\bb{S}^1 \x \R$ equipped with the product metric.
The equations of motion are given by Hamilton's equations, which are parameterized by the Hamiltonian $H : T^*\bb{S}^1 \-> \R$.
To learn the equations of motion from observed data, we place a GP prior on the Hamiltonian, with covariance given by a squared exponential kernel on the cylinder, defined as a product kernel of squared exponential kernels on the circle and real line.
Following \textcite{hensman13}, training proceeds using mini-batch stochastic variational inference with automatic relevance determination. 
The full setup is given in Appendix \ref{apdx:experiments}.

To generate trajectories from the learned equations of motion, following \textcite{wilson20}, we approximate the prior GP using Fourier features, and employ \eqref{eqn:pathwise} to transform prior sample paths into posterior sample paths.
We then generate trajectories by solving the learned Hamilton's equations numerically for each sample, which is straightforward because the approximate posterior is a basis function approximation and therefore easily differentiated in the ordinary deterministic manner.
Results can be seen in Figure \ref{fig:cylinder}.
From these, we see that our GP learns the correct qualitative behavior of the equations of motion, mirroring the results of \textcite{deisenroth11}.

\subsection{Sample path visualization}

To understand how complicated geometry affects posterior uncertainty estimates and illustrate the techniques on a general Riemannian manifold, we consider a posterior sample path visualization task.
We take $M$ to be the \emph{dragon} manifold from the Stanford 3D scanning repository, modified slightly to remove components not connected to the outer surface.
We represent the manifold using a $202490$-triangle mesh and obtain 500 Laplace--Beltrami eigenpairs numerically using the \emph{Firedrake} package \cite{rathgeber16}.

For training data, we introduce a ground truth function by fixing a distinguished point at the end of the dragon's snout, and compute the sine of the geodesic distance from that point.
We then observe this function at $52$ points on the manifold chosen from the mesh's nodes, and train a Mat\'{e}rn GP regression model with smoothness $\nu = 3/2$ by maximizing the marginal likelihood with respect to the remaining kernel hyperparameters.
By using the path-wise sampling expression \eqref{eqn:pathwise}, we obtain posterior samples defined on the entire mesh.

Results can be seen in Figure \ref{fig:dragon}.
Here, we see that posterior mean and uncertainty estimates match the manifold's shape seamlessly, decaying roughly in proportion with the geodesic distance in most regions.
In particular, we see that the two sides of the dragon's snout have very different uncertainty values, despite close Euclidean proximity.
This mimics the well-known \emph{swiss roll} example of manifold learning \cite[Section 6.1.1]{lee2007}, and highlights the value of using a model which incorporates geometry.

\begin{figure}
\begin{minipage}{0.01\textwidth}
\begin{subfigure}{\linewidth}
  \centering
  \includegraphics[width=\linewidth]{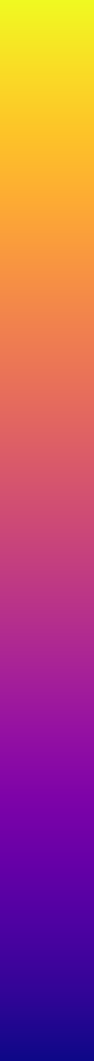}  
  \caption*{} 
  \label{fig:dragon-colorbar}
\end{subfigure}    
\end{minipage}
\begin{minipage}{.95\textwidth}
\begin{subfigure}{.24\linewidth}
  \centering
  \includegraphics[width=\linewidth, trim=0 0 0 0, clip]{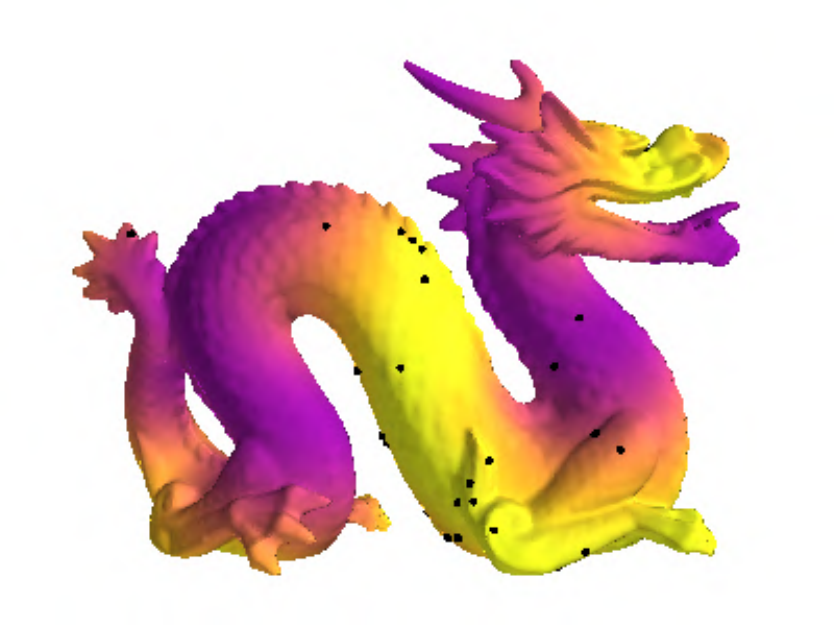}  
  \caption{Ground truth}
  \label{fig:dragon-ground-truth}
\end{subfigure}
\begin{subfigure}{.24\linewidth}
  \centering
  \includegraphics[width=\linewidth, trim=0 0 0 0, clip]{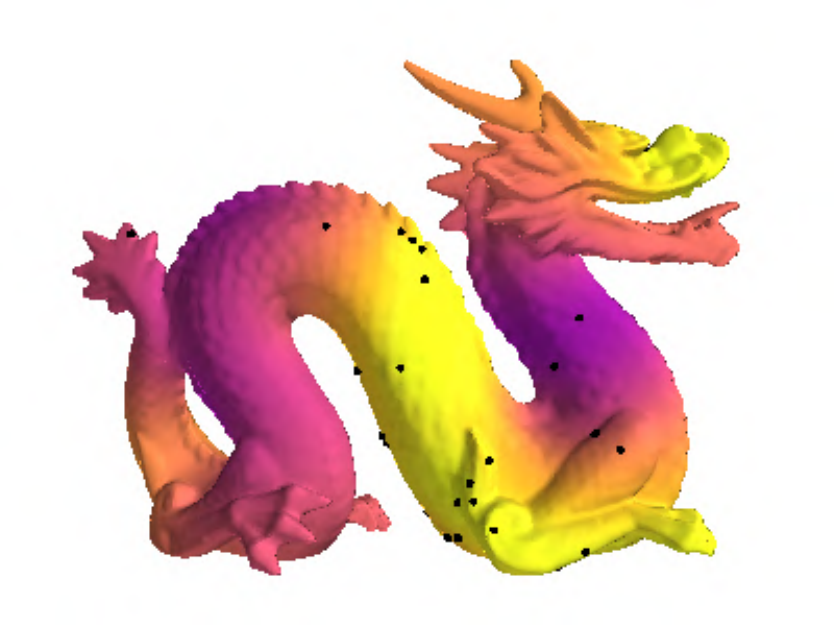}  
  \caption{Mean}
  \label{fig:dragon-mean}
\end{subfigure}
\begin{subfigure}{.24\linewidth}
  \centering
  \includegraphics[width=\linewidth, trim=0 0 0 0, clip]{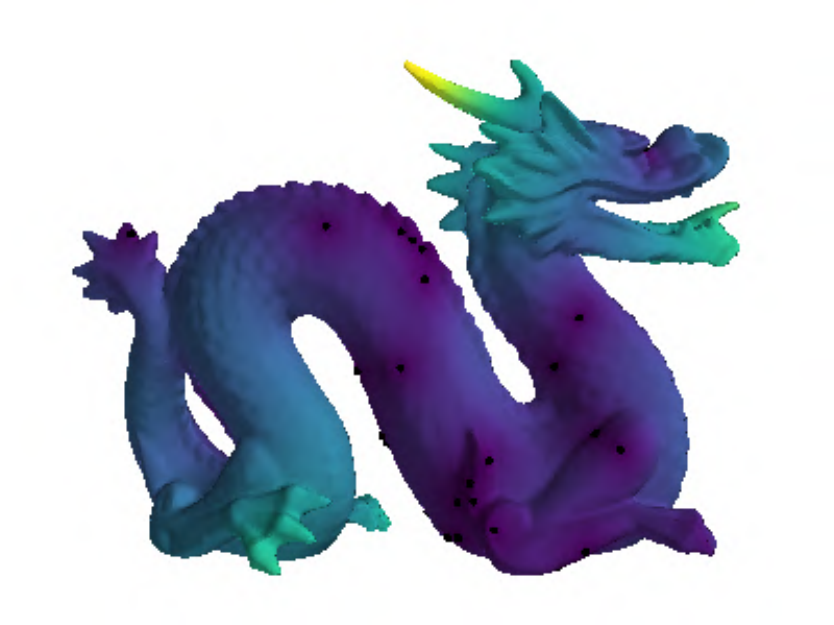}  
  \caption{Standard deviation}
  \label{fig:dragon-sample-0}
\end{subfigure}
\begin{subfigure}{.24\linewidth}
  \centering
  \includegraphics[width=\linewidth, trim=0 0 0 0, clip]{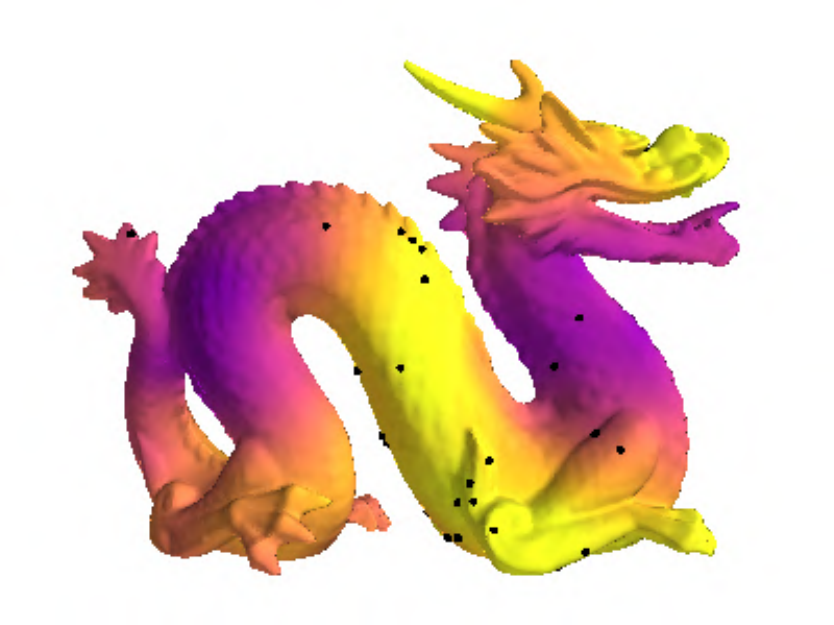}
  \caption{One posterior sample}
  \label{fig:dragon-sample-1}
\end{subfigure}
\end{minipage}
\begin{minipage}{0.01\textwidth}
\begin{subfigure}{\linewidth}
  \centering
  \includegraphics[width=\linewidth]{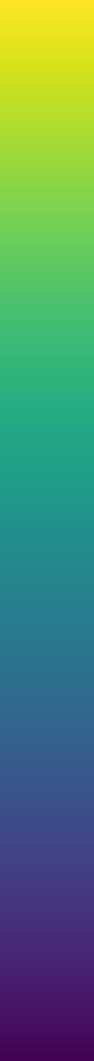}  
  \caption*{} 
  \label{fig:dragon-colorbar-right}
\end{subfigure}    
\end{minipage}
\caption{Visualization of a Mat\'{e}rn Gaussian process posterior on the dragon.
We plot the true function values, posterior mean, marginal posterior variance, and one posterior sample evaluated on the entire mesh.
Here, black dots denote training locations, and color represents value of the corresponding functions.
Additional posterior samples can be seen in Appendix \ref{apdx:experiments}.
}
\label{fig:dragon}
\end{figure}

\section{Conclusion}

In this work, we developed techniques for computing the kernel, spectral measure, and Fourier feature approximation of Mat\'{e}rn and squared exponential Gaussian processes on compact Riemannian manifolds, thereby constructively generalizing standard Gaussian process techniques to this setting.
This was done by viewing the Gaussian processes as solutions of stochastic partial differential equations, and expressing the objects of interest in terms of Laplace--Beltrami eigenvalues and eigenfunctions.
The theory was demonstrated on a set of simple examples: learning the equations of motion of an ideal pendulum, and sample path visualization for a Gaussian process defined on a dragon.
This illustrates the theory in settings both where Laplace--Beltrami eigenfunctions have a known analytic form, and where they need to be calculated numerically using a differential equation or graphics processing framework.
Our work removes limitations of previous approaches, allowing Mat\'{e}rn and squared exponential Gaussian processes to be deployed in mini-batch, online, and non-conjugate settings using variational inference.
We hope these contributions enable practitioners in robotics and other physical sciences to more easily incorporate geometry into their models.

\section*{Broader Impact}

This is a purely theoretical paper. 
We develop technical tools that make Mat\'{e}rn Gaussian processes easier to work with in the Riemannian setting.
This enables practitioners who are not experts in stochastic partial differential equations to model data that lives on spaces such as spheres and tori.

We envision the impact of this work to be concentrated in the physical sciences, where spaces of this type occur naturally.
Since the state spaces of most robotic arms are Riemannian manifolds, we expect these ideas to improve performance of model-based reinforcement learning by making it easier to incorporate geometric prior information into models.

Since climate science is concerned with studying the globe, we also expect that our ideas can be used to model environmental phenomena, such as sea surface temperatures.
By employing Gaussian processes for data assimilation and building them into larger frameworks, this could facilitate more accurate climate models compared to current methods.

These impacts carry forward to potential generalizations of our work.
We encourage practitioners to consider impacts on their respective disciplines that arise from incorporating geometry into models.

\section*{Acknowledgments and Disclosure of Funding}

VB was supported by the St. Petersburg Department of Steklov Mathematical Institute of Russian Academy of Sciences and by the Ministry of Science and Higher Education of the Russian Federation, agreement N\textsuperscript{\underline{o}} 075-15-2019-1620.
PM was supported by the Ministry of Science and Higher Education of the Russian Federation, agreement N\textsuperscript{\underline{o}} 075-15-2019-1619.
VB and PM were supported by "Native towns", a social investment program of PJSC Gazprom Neft, and by the Department of Mathematics and Computer Science of St. Petersburg State University.
AT was supported by the Department of Mathematics at Imperial College London.

\printbibliography

\newpage
\appendix

\section{Additional experimental details}
\label{apdx:experiments}

\subsection*{Sample path visualization}

Here, we further explore the example of a Gaussian process on the dragon manifold. 
Figure \ref{fig:dragon_appendix} presents nine additional samples from Gaussian process posterior. 
Note that the we change the color palette in order to cover samples' value range. 
We repeat the first row of Figure \ref{fig:dragon_appendix} with the new color palette.

To define a Mat\'{e}rn kernel on the dragon manifold, we employ a re-parametrized version of  \eqref{eqn:mani-matern-formula}, which is
\[
    k_{\nu}(x, x') = \sigma^2\sum_{n=0}^\infty \del{\frac{1}{\kappa^2} + \lambda_n}^{-\nu-\frac{d}{2}} f_n(x)f_n(x')
    .
\]
Here, we need to compute eigenvalues and eigenfunctions of the Laplace--Beltrami operator. 
We do so using the \emph{Galerkin finite element method} (FEM), and approximate the manifold as a triangular mesh with $K=100179$ vertices.
This involves solving a Helmholtz equation, which is a far easier problem than solving the SPDE \ref{eqn:spde-matern}, since, among other reasons, the equation is a standard deterministic second-order linear PDE that only needs to be solved once, rather than once-per-sample.
Each vertex $v_k$, $k = 1, \ldots, K$ is associated with a piecewise-linear basis function $\phi_k$, such that $\phi_k(v_l) = \delta_{kl}$ where $\delta$ is the Kronecker delta.
This leads to a \emph{discrete Laplace-Beltrami operator} $\lap$, as a discretization of the Laplace-Beltrami $\lap_g$ on the manifold.
Finally, the eigenproblem is stated as follows: find $\lambda_n$, $f_n$,  such that
\<
    \label{eqn:discrete-laplace-eigenproblem}
    \innerprod{\lap f_n}{\phi_k} &= \lambda_n \innerprod{f_n}{\phi_k}
    &
    k &= 1, \ldots, K
\>
where we regard the functions $f_n$, $\phi_k$ as $K$-dimensional vectors: $(f_n)_k = f_n(v_k)$. 
Since the resulting eigenproblem \eqref{eqn:discrete-laplace-eigenproblem} is finite-dimensional, it can be solved using standard numerical approaches. 
Note that the discrete Laplacian cannot have more than $K$ eigenvalues. In our experiments, we use the Firedrake software package \cite{rathgeber16}, which provides a high-level domain-specific language for computing discrete Laplacians and related tasks.
We use Arnoldi method with shift-invert spectral transform to compute the first $N=500$ (smallest magnitude) eigenvalues and corresponding eigenfunctions on the dragon mesh, using numerical routines from the PETSc package, which Firedrake calls. 
This allows us to approximate the formula \eqref{eqn:mani-matern-formula} with the first $N$ components of the sum, given by
\[
    k_{\nu}(x, x') \approx \sigma^2\sum_{n=0}^N \del{\frac{1}{\kappa^2} + \lambda_n}^{-\nu-\frac{d}{2}} f_n(x)f_n(x').
\]

The formula also leads to a Fourier approximation of the prior:
\<
    f(x) &\approx \sigma \sum_{n=0}^N w_n \del{\frac{1}{\kappa^2} + \lambda_n}^{-\frac{\nu}{2}-\frac{d}{4}} f_n(x)
    &
    w_n &\~[N](0,1).
\>

Once these expressions are obtained, standard GP training techniques are utilized to compute the posterior distribution using path-wise sampling, given by equation \eqref{eqn:pathwise}. In the experiments we set the smoothness parameter $\nu$ to be $3/2$, and Gaussian noise variance to be $10^{-15}$.
We obtain $\sigma^2$ and $\kappa$ by gradient descent optimization of the marginal likelihood.

\newpage

\begin{figure}[t]
\begin{minipage}{\textwidth}
\begin{minipage}{0.01\textwidth}
\begin{subfigure}[t]{\linewidth}
  \centering
  \includegraphics[width=\linewidth]{Figures/plasma.pdf}  
  \caption*{} 
\end{subfigure}    
\end{minipage}
\begin{minipage}{.97\textwidth}
\begin{subfigure}{.32\linewidth}
  \centering
  \includegraphics[width=\linewidth, trim=0 0 0 0, clip]{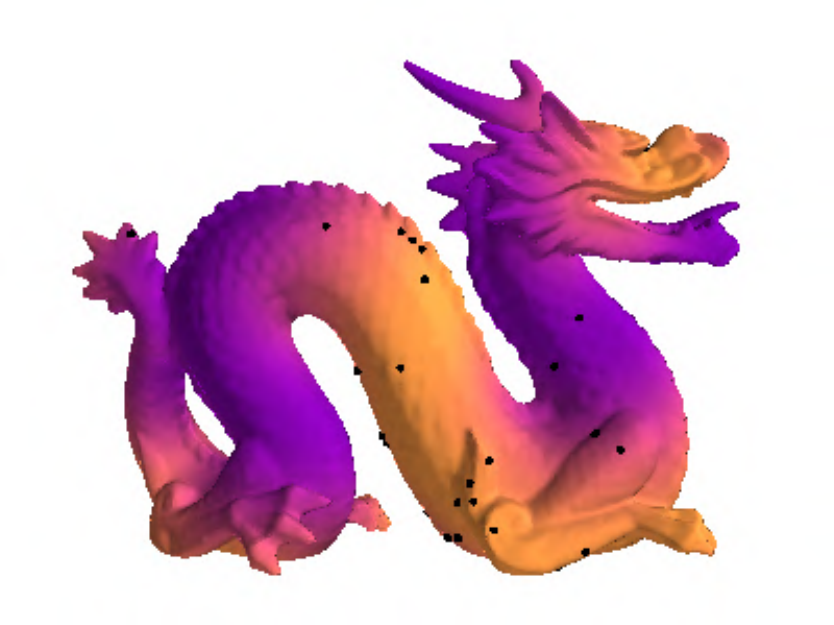}  
  \caption{Ground truth}
\end{subfigure}
\begin{subfigure}{.32\linewidth}
  \centering
  \includegraphics[width=\linewidth, trim=0 0 0 0, clip]{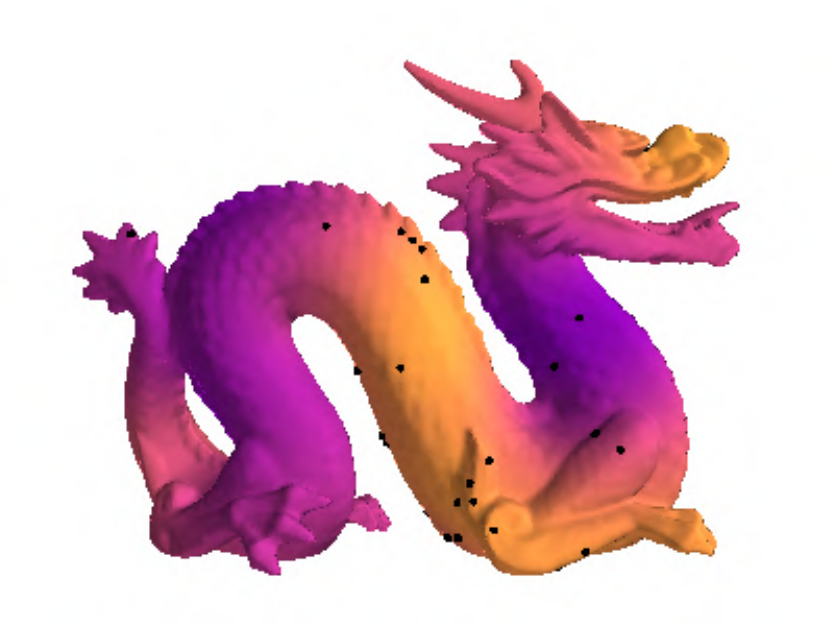}  
  \caption{Mean}
\end{subfigure}
\begin{subfigure}{.32\linewidth}
  \centering
  \includegraphics[width=\linewidth, trim=0 0 0 0, clip]{Figures/dragon/1_standard_deviation.pdf}  
  \caption{Standard deviation}
\end{subfigure}
\end{minipage}
\begin{minipage}{0.01\textwidth}
\begin{subfigure}{\linewidth}
  \centering
  \includegraphics[width=\linewidth]{Figures/viridis.pdf}  
  \caption*{} 
\end{subfigure}    
\end{minipage}
\end{minipage}
\newline
\includegraphics[width=.32\linewidth]{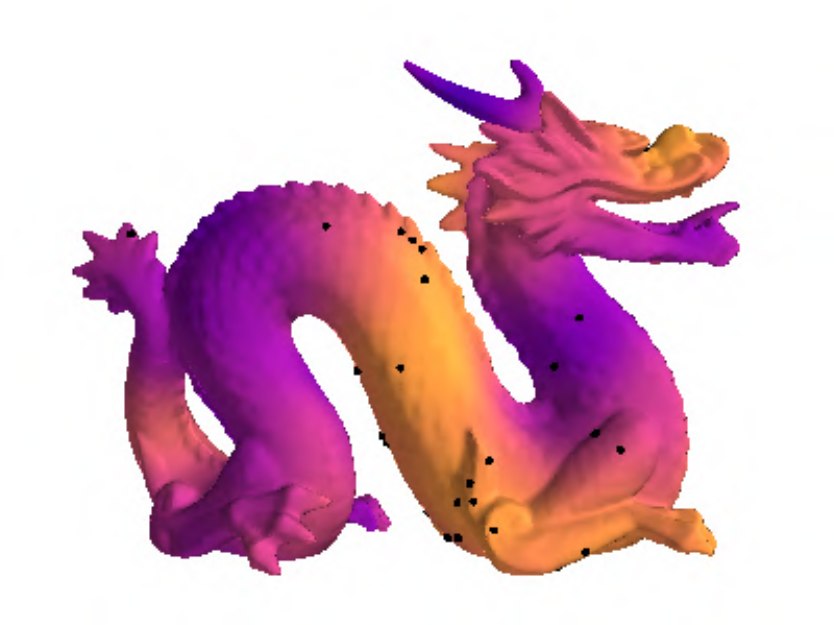}
\includegraphics[width=.32\linewidth]{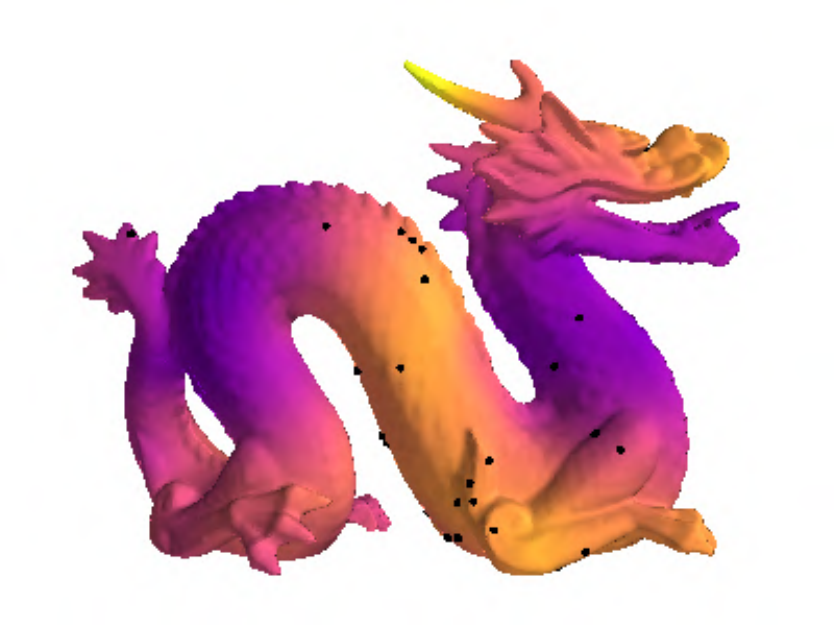}
\includegraphics[width=.32\linewidth]{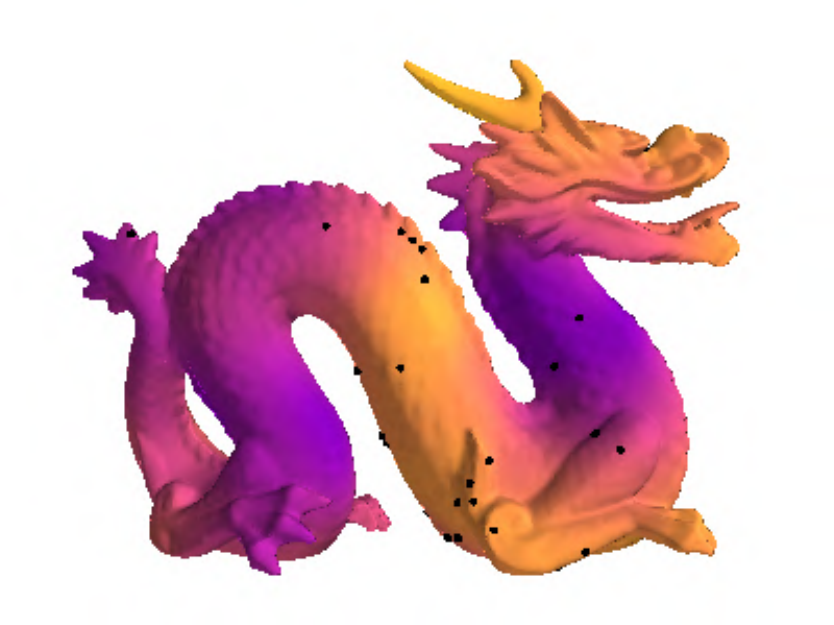}
\newline
\includegraphics[width=.32\linewidth]{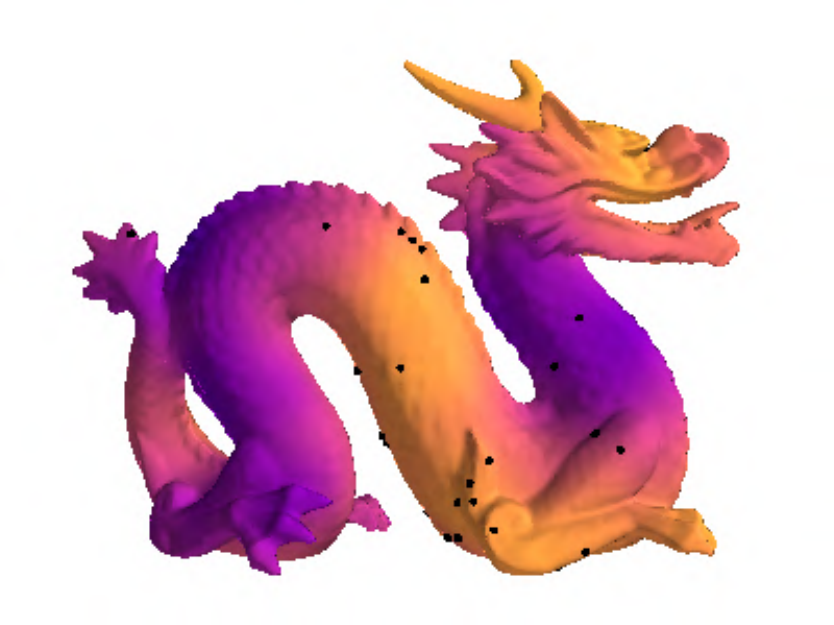}
\includegraphics[width=.32\linewidth]{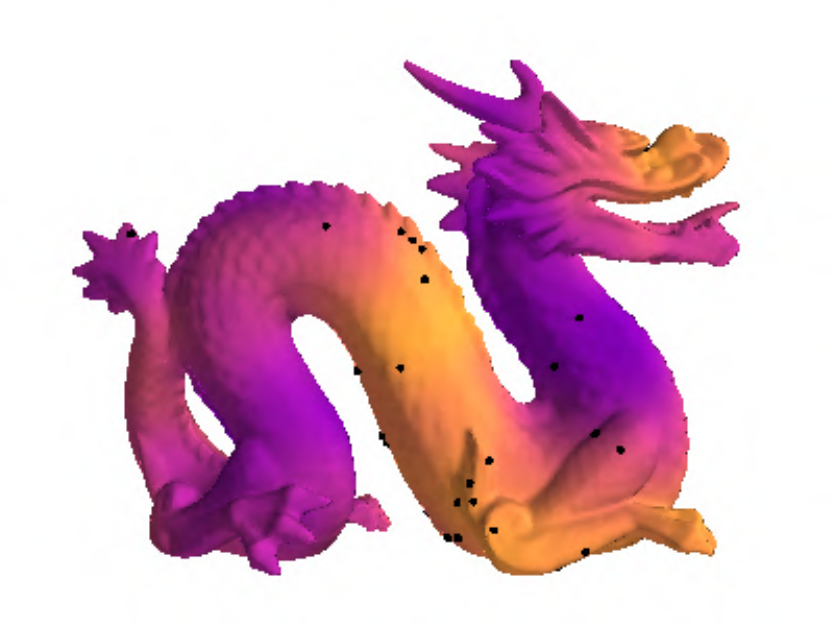}
\includegraphics[width=.32\linewidth]{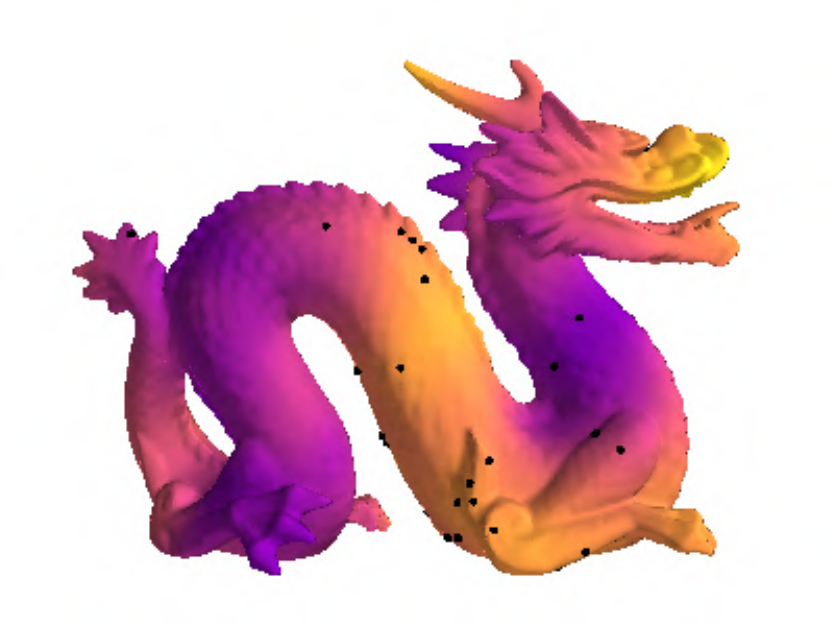}
\newline
\includegraphics[width=.32\linewidth]{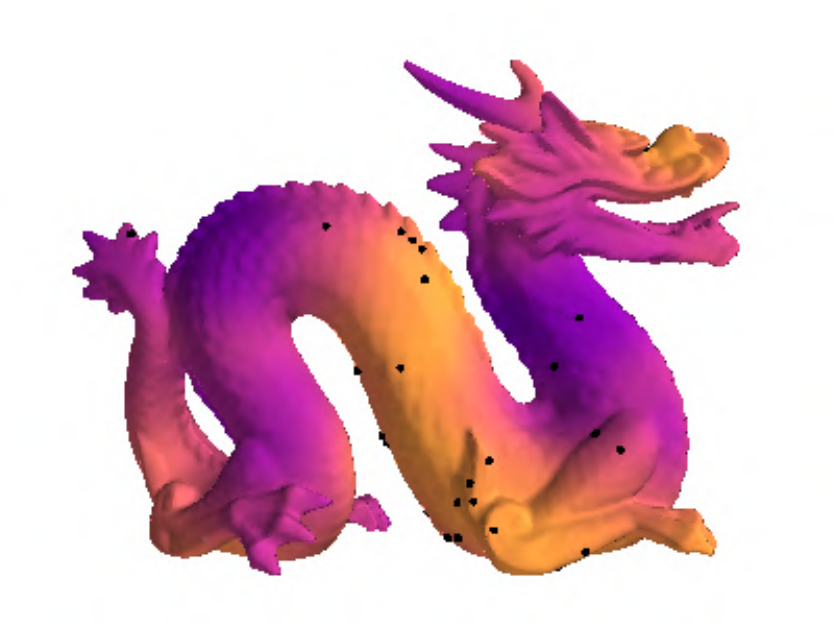}
\includegraphics[width=.32\linewidth]{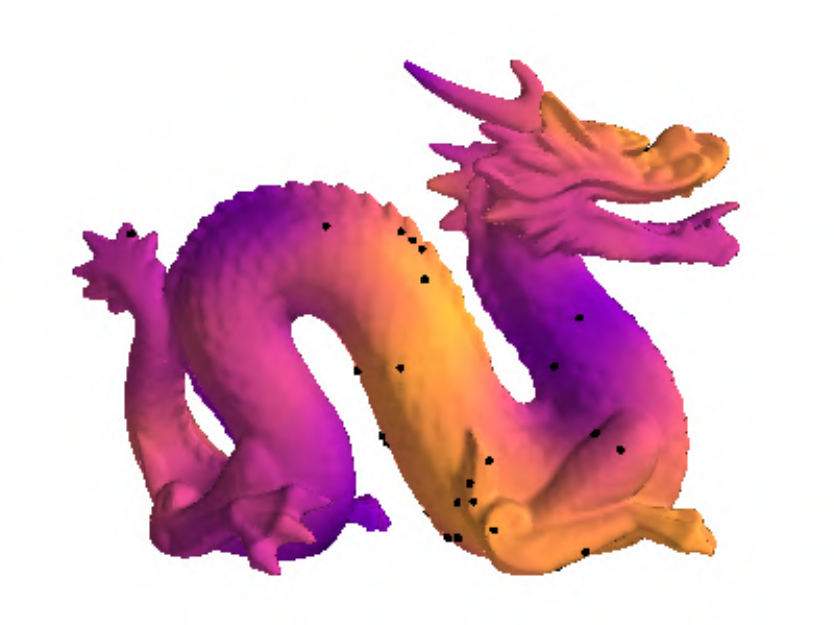}
\includegraphics[width=.32\linewidth]{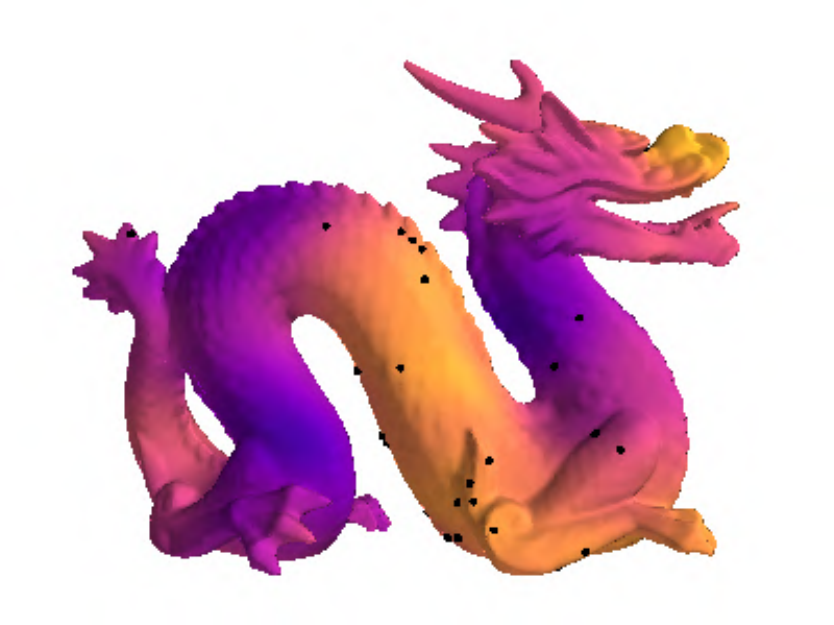}
\caption{Visualization of a Mat\'{e}rn Gaussian process posterior on the dragon.
We plot the true function values, posterior mean, marginal posterior standard deviation, and nine random function draws from the posterior.
Here, black dots denote training locations, and color represents value of the corresponding functions.
The color palette is changed slightly compared to Figure \ref{fig:dragon} in order to represent the range of the samples more effectively.
}
\label{fig:dragon_appendix}
\end{figure}

\clearpage

\subsection*{Dynamical system prediction}

Here we describe the setup in the dynamical systems predictions experiment in more detail. 
Our system is an ideal pendulum, parameterized by angle and angular momentum which lie on the cylinder~$\bb{S}^1 \x \R^1$, and which we denote by $(\theta, p_\theta)$.
The true equations of motion are given by Hamilton's equations
\<
\dot\theta &= \pd{H}{p_\theta}
&
\dot{p} &= - \pd{H}{\theta}
\>
with
\[
H(\theta, p_\theta) = \frac{p_\theta^2}{2 m l^2} + mgl(1-\cos(\theta))
\]
where $(m,g,l)$ are the mass, gravitational constant, and length of the pendulum.
We set $m = 1$, $g = 9.8$, $l = 2$.

Training data is obtained as follows.
We do not observe the Hamiltonian: instead, we observe its partial derivative pairs $(\pd{H}{\theta}, \pd{H}{p_\theta})$.
In a reinforcement learning setting, following \textcite{deisenroth11}, these can be obtained by backward integration of observed trajectories.
In our simplified setting, we generate training data by computing said partial derivatives at random locations, sampled uniformly on the rectangle $(0,2\pi) \x (-20,20)$, generating $1000$ total training points.

To obtain the model, we first place a Gaussian process prior directly on the Hamiltonian.
To ensure that (a) $\theta$ is supported on $[0,2\pi)$, rather than $[0,1)$, and (b) a unified random Fourier feature expansion for the prior is possible for both the $\theta$ and $p_\theta$ components, we use a re-parameterized form for the kernel and spectral measure.
These are given by
\<
k_\theta(\theta,\theta') &= \sum_{n\in2\pi\Z^d} \exp\del{-\norm{\frac{\theta - \theta' + n}{2\pi \sqrt{2}\kappa}}^2} = \sum_{n\in2\pi\Z^d} \exp\del{-\norm{\frac{\theta - \theta' + n}{\kappa_\theta}}^2}
\\
\rho_\theta(n) &= \sqrt{2\pi} 2^{-3/2}\pi^{-1}\kappa_\theta \exp(-2\pi^2 2^{-3}\pi^{-2}\kappa_\theta^2n^2) = 2^{-1} \pi^{-1/2} \kappa_\theta \exp(-2^{-2}\kappa_\theta^2n^2)
\>
which from the kernel and spectral measure introduced in the manuscript by defining $\theta = 2\pi x, \theta' = 2\pi x'$, and $\kappa_\theta = 2^{3/2}\pi\kappa$ so that $\kappa = 2^{-3/2}\pi^{-1}\kappa_\theta$.
For the $p_\theta$ component, we use the re-normalized squared exponential kernel
\[
k_{p_\theta}(p_\theta, p_\theta') = \exp\del{- \norm{\frac{p_\theta - p_\theta'}{\kappa_{p_\theta}}}^2}
\]
and denote the corresponding spectral measure by $\rho_{p_\theta}$. 
The full kernel on $(\theta, p_\theta)$ is given by 
\<
k\del[1]{(\theta,p_\theta),(\theta',p_\theta')} &= \sigma^2 k_\theta(\theta,\theta')k_{p_\theta}(p_\theta, p_\theta')
.
\>
Similarly, denote the full spectral measure over $\Z \x \R$ by $\rho$.
Sampling from the posterior is performed by sampling from the prior using a random Fourier feature approximation and transforming the resulting draws into posterior draws using \eqref{eqn:pathwise}.

Unfortunately, since the spectral measure for this kernel is the product of a discrete measure for the $\theta$ component, and absolutely continuous measure for the $p_\theta$ component, the resulting optimization objective is not (automatically) differentiable with respect to $\kappa_{p_\theta}$.
To enable use of automatic relevance determination, we develop an importance-sampling-based \emph{reparametrization trick} by employing the generalized random Fourier feature expansion
\[
f(\theta,p_\theta) \approx \sigma \sqrt{\frac{2}{\ell}} \sum_{j=1}^\ell \gamma_j w_j \cos\del{\innerprod{\v\omega_j}{\frac{(\theta,p_\theta)}{\v\lambda}} + \beta_j}
\]
where division by $\v\lambda = (1,\kappa_{p_\theta})$ is performed element-wise, and
\<
\v\omega_j &\~ \widehat\rho
&
\beta_j &\~[U](0,2\pi)
&
w_j &\~[N](0,1)
\>
where $\widehat\rho$ is the \emph{standard spectral measure}, which is equal to $\rho$ except with $\kappa_\theta$ and $\kappa_{p_\theta}$ fixed to reference values, in our case $\kappa_\theta = \kappa_{p_\theta} = 1$.
The importance weights $\gamma_j$ are given by
\[
\gamma_j = \sqrt{\frac{\rho_\theta(\omega_{j\theta}) / C_\theta}{\widehat\rho_\theta(\omega_{j\theta}) / \widehat{C}_\theta}} = \sqrt{\frac{\rho_\theta(\omega_{j\theta}) \widehat{C}_\theta}{\widehat\rho_\theta(\omega_{j\theta}) C_\theta}}
\]
where $\omega_{j\theta}$ is the $\theta$-component of $\v\omega_j$ and $C_\theta = \sum_{n\in\Z}\rho_\theta(n)$ and $\widehat{C}_\theta = \sum_{n\in\Z}\widehat\rho_\theta(n)$ are their respective normalizing constants.
Using this more general random Fourier feature approximation, the training objective becomes differentiable with respect to $\kappa_{p_\theta}$.

Since we do not observe the Hamiltonian, but rather its partial derivatives $(\pd{H}{\theta}, \pd{H}{p_\theta})$, as our full model we employ the \emph{gradient} of the Gaussian process developed above, which yields a vector-valued Gaussian process.
The kernel of said process is obtained by differentiating the kernels above.

To complete the model, we now introduce the inducing point approximation.
We use a total of 35 vector-valued inducing points, which are initialized on an evenly-spaced grid over the domain of the training data. 
For the prior approximation, we use a total of 128 random Fourier features.

Following \textcite{titsias09a} and \textcite{hensman13}, training proceeds by minimizing Kullback-Leibler divergence between the inducing point GP and the true posterior GP. 
We optimize the inducing points, inducing covariance, and all model hyperparameters.
For the loss, in addition to the KL divergence, we include $\ell^2$ regularization terms corresponding to log-normal hyperpriors for the hyperparameters.
For the kernel, these are given as $\sigma^2 \~[LN](0,1)$, $\kappa_\theta \~[LN](0,1)$ and $\ln\kappa_{p_\theta} \~[LN](1.5,1)$.
For the GP, the error variance hyperprior $\tau^2 \~[LN] (10^{-12}, 1)$.
All parameters are initialized at their hyperprior's mean.
The jitter term is set to $\varsigma = 10^{-5}$.

Optimization is performed by the ADAM algorithm, with learning rate set to $\eta = 0.01$ and default values for the other hyperparameters.
We use a mini-batch size of 128, and train until convergence.

To generate trajectories of the dynamical system under the learned Hamiltonian, following \textcite{wilson20}, we use \eqref{eqn:pathwise} to draw a set of basis coefficients from the posterior distribution, and form a basis function approximation of our posterior GP.
We plug this function back into Hamilton's equations, and solve them numerically by employing a St\"{o}rmer-Verlet integrator.
The step size is tuned for each initial condition to ensure all trajectories in Figure \ref{fig:cylinder} cross each other on the rear side of the cylinder at approximately the same time when using the true Hamiltonian after 50 time steps, and range from $0.02$ to $0.031$.
These step sizes are then used to produce error bars for the learned Hamiltonian.

To generate the error bars on the cylinder Figure \ref{fig:cylinder}, we first compute the mean trajectory under the GP model for each time step.
Then, for each time step, we project the trajectories onto the tangent plane on the cylinder located at the mean, using the cross product identity.
In this tangent plane, we then project the trajectory points onto a line perpendicular to the tangent vector pointing in the direction of the mean trajectory obtained by backwards integration.
We calculate 95\% intervals over this line, and plot them projected back from the tangent plane onto the surface of the cylinder.
The error bars for the positions and momenta of the distinguished trajectory on the right-hand-side of Figure \ref{fig:cylinder}, which are not plotted on the surface of the cylinder, are obtained by re-parameterizing $\vartheta = ((\theta + \pi) \mod 2\pi) - \pi$ to ensure $\vartheta \in [-\pi,\pi)$, and calculating 2.5\% and 97.5\% quantiles in the standard way.

\section{Additional examples and expressions} \label{apdx:examples}

\subsection*{Circle}
Here we discuss closed-form expressions for Mat\'{e}rn and squared exponential kernels on circle~$\bb{S}^1 = \bb{T}$. These kernels are given in \eqref{eqn:matern-torus} and \eqref{eqn:rbf-torus} respectively, with $d=1$ in our setting.
Applying the generalized Poisson summation formula \cite[Chapter VIII]{stein2016} to these expressions gives
\<
&k_{\nu}(x, x')
=
\sum\limits_{n \in \Z}
\frac{S_{\nu}(n)}{C'_\nu}
e^{2 \pi i n (x - x')}
,
&
&k_{\infty}(x, x')
=
\sum\limits_{n \in \Z}
\frac{S_{\infty}(n)}{C'_\infty}
e^{2 \pi i n (x - x')}
,
\>
where $S_\nu$ and $S_{\infty}$ are precisely the spectral densities of the standard Mat\'{e}rn and squared exponential kernels over $\R$. 
The specific formulas for $S_\nu, S_{\infty}$ are given in \textcite[Section 4.2.1]{rasmussen06}:
\<
S_{\nu}(\xi)
&=
\sigma^2
\frac{
    2
    \pi^{\frac{1}2}
    \Gamma(\nu+\frac{1}2)
    (2\nu)^\nu
}{
    \Gamma(\nu)\kappa^{2\nu}
}
\left(
    \frac{2\nu}{\kappa^2}
    +
    4\pi^2 \xi^2
\right)^{-\left(\nu+\frac{1}2\right)}
,
\\
S_{\infty}(\xi)
&=
\sigma^2
(2 \pi \kappa^2)^{1/2}
e^{-2 \pi^2 \kappa^2 \xi^2}
,
\>
where the cumbersome constants ensure that the original GP over $\R$ has variance equal to $\sigma^2$.
Periodic summation does not preserve variance, thus requiring additional constants $C'_{(\cdot)}$ to recover variance $\sigma^2$.
This makes the original constants redundant, so we instead consider
\<
\tilde{k}_{\nu}(x, x')
&=
\sum\limits_{n \in \Z}
\left(
    \frac{2\nu}{\kappa^2}
    +
    4\pi^2 n^2
\right)^{-\left(\nu+\frac{1}2\right)}
e^{2 \pi i n \cdot (x - x')}
,
\\
\tilde{k}_{\infty}(x, x')
&=
\sum\limits_{n \in \Z}
e^{-2 \pi^2 \kappa^2 n^2}
e^{2 \pi i n \cdot (x - x')}
.
\>
For $\nu = \infty$ the right-hand side is precisely one of the classical Jacobi theta functions, $\vartheta_3(z, q)$ (see definition in \textcite[equation 16.27.3]{abramowitz1972}), with parameters $z = \pi (x - x')$ and $q = \exp(-2 \pi^2 \kappa^2)$, giving
\[
\tilde{k}_{\infty}(x, x')
=
\vartheta_3(\pi(x - x'), \exp(-2 \pi^2 \kappa^2))
.
\]
To obtain $k_\infty$ from $\tilde{k}_\infty$ we need to find $C_\infty$ such that $\tilde{k}_\infty(x,x) / C_\infty = \sigma^2$. 
Obviously, $C_\infty = \tilde{k}_\infty(x,x)/ \sigma^2$, where the right hand side does not depend on $x$, so the constant is well-defined. Hence
\[
k_{\infty}(x, x')
=
\frac{\sigma^2}{\vartheta_3(0, \exp(-2 \pi^2 \kappa^2))} \vartheta_3(\pi(x - x'), \exp(-2 \pi^2 \kappa^2))
.
\]
Returning to the periodic summation of the original normalized kernel, we obtain
\[
C'_\infty = C_\infty \sigma^2 (2 \pi \kappa^2)^{1/2} = \vartheta_3(0, \exp(-2 \pi^2 \kappa^2)) (2 \pi \kappa^2)^{1/2}
.
\]

Summarizing, we obtain the following.
\begin{example}[Squared exponential kernel on $\bb{S}^1$] The squared exponential kernel, normalized to have variance $\sigma^2$, and the corresponding spectral density, are given by
\<
k_{\infty}(x, x')
&=
\frac{\sigma^2}{\vartheta_3(0, \exp(-2 \pi^2 \kappa^2))}
\vartheta_3(\pi(x - x'), \exp(-2 \pi^2 \kappa^2))
,
\\
\rho_{\infty}(n)
&=
\frac{\sigma^2}{\vartheta_3(0, \exp(-2 \pi^2 \kappa^2))}
\exp(-2 \pi^2 \kappa^2 n^2), \qquad n \in \Z
.
\>
\end{example}

Now we turn our attention to kernels $k_{\nu}$. 
After an appropriate \textit{mutatis mutandis} applied to the closed-form of the Fourier series 
\[
\frac{ \alpha \sinh(\alpha \pi) } { \pi }
\sum\limits_{k=-\infty}^{\infty} \frac{ \exp(i k \theta) } { (\alpha^2 + k^2)^n }
\]
provided in the supplementary material of \textcite{guinness16} we get the following.

\begin{example}[Mat\'{e}rn kernel on $\bb{S}$ for half-integer $\nu$]
Let $\nu = 1/2 + s$, $s \in \N$. 
The Mat\'{e}rn kernel, normalized to have variance $\sigma^2$, and the corresponding spectral density, are given by
\<
k_{\nu}(x, x^\prime)
&=
    \frac{\sigma^2}{C_\nu}
    \sum_{k=0}^{s} a_{s,k} \left( \sqrt{2\nu} \cdot \frac{\abs{x-x'} - 1/2}{\kappa} \right)^{k}
        \hyp^{k} \left( \sqrt{2\nu} \cdot \frac{\abs{x-x'} - 1/2}{\kappa} \right)
\\
\rho_{\nu}(n)
&=
\frac{
    2 \sigma^2
    \sqrt{2\nu}
    \sinh \left(\frac{\sqrt{2\nu}}{2 \kappa} \right)
    }
    {C_\nu
    (2\pi)^{1 - 2 \nu}
    \kappa
    }
\del{\frac{2\nu}{\kappa^2} + 4\pi^2 n^2 }^{-\nu-1/2}
, \qquad n \in \Z
\>
where various components of the expression are defined as follows.
\1 $\hyp^{k}(\cdot)$ is defined as $\cosh(\cdot)$ when $k$ is even and $\sinh(\cdot)$ when $k$ is odd.
\2 $C_\nu$ is chosen so that $k_\nu(x, x) = \sigma^2$.
\3 $a_{s,k}$ are constants defined as follows, following a modification of the derivation given by \textcite{guinness16}.
\1 First, for the special case $k=s$, define
\<
a_{s,s} = \left( \del{-\frac{\nu}{\pi^2 \kappa^2}}^{s} (s)! \right)^{-1}
.
\>
\2 Next, define the constants $h_{rk}$ as 
\[
h_{rk} = \sum_{j=0}^{2r+1} \binom{2r+1}{j} (k)_j 
    \left(\frac{\sqrt{2\nu}}{2 \kappa}  \right)^{k-j} \cdot \hyp^{k-j+1} \left(
            \frac{\sqrt{2\nu}}{2 \kappa} \right)           
\]
for $r = 0, \ldots, s-1$ and $k = 0, \ldots, s$, where $(k)_j$ is the falling factorial
\[
(k)_j =
\begin{cases}
1, &\t{when} j = 0, \\
0, &\t{when} j > k, \\
k (k-1) \dots (k-j+1), &\t{otherwise.}
\end{cases}
\]
\3 Finally, define the matrix 
\<
\mathbf{H}_s = (h_{rk})_{r = 0, \ldots, s-1}^{k = 0, \ldots, s-1}
,
\>
and a vector $\mathbf{h}_s = [h_{0,s}, \ldots h_{s-1,s}]^\top$. 
Then the remaining constants $a_{s,k}$ for $k\neq s$ are given as
\<
[a_{s,0}, \ldots, a_{s,s-1}]^\top = -a_{s,s} \mathbf{H}_s^{-1} \mathbf{h}_{s}
.
\>
\0
\0 

For $\nu = 1/2$ the above formulae reduce to
\<\label{eq:matern12_circle}
k_{1/2}(x, x')
&=
\frac{\sigma^2}{\cosh \del{ \frac{1}{2 \kappa} }}
\cosh \del{ \frac{\abs{x-x'} - 1/2}{\kappa} }
,
\\
\rho_{1/2}(n)
&=
\frac{
    \sigma^2
    2
    \sinh \left(\frac{1}{2 \kappa} \right)
    }
    {
    \kappa
    \cosh \del{ \frac{1}{2 \kappa} }
    }
\del{\frac{1}{\kappa^2} + 4\pi^2 n^2 }^{-1}
.
\>
For $\nu = 3/2$,
\<
k_{3/2}(x, x')
&=
\frac{\sigma^2}{C_{3/2}}
\del{
    \frac{\pi^2 \kappa}{3}
    \del{
        2 \kappa
        +
        \sqrt{3}
        \coth \del{
            \frac{\sqrt{3}}{2 \kappa}
        }
        }
    \cosh \del{u}
    -
    \frac{2 \pi^2 \kappa^2}{3}
    u
    \sinh \del{u}
}
,
\\
\rho_{3/2}(n)
&=
\frac{\sigma^2}{C_{3/2}}
\frac{
    2
    \sqrt{3}
    \sinh \left(\frac{\sqrt{3}}{2 \kappa} \right)
    }
    {
    (2\pi)^{-2}
    \kappa
    }
\del{\frac{3}{\kappa^2} + 4\pi^2 n^2 }^{-2}
,
\>
where $u = \sqrt{3} \frac{\abs{x-x'}-1/2}{\kappa}$.

For $\nu = 5/2$, they reduce to
\<
k_{5/2}(x, x')
&=
\frac{\sigma^2}{C_{5/2}}
\del{
    a_{2, 0}
    \cosh \del{u}
    +
    a_{2, 1}
    u
    \sinh \del{u}
    +
    a_{2, 2}
    u^2
    \cosh \del{u}
}
,
\\
\rho_{5/2}(n)
&=
\frac{\sigma^2}{C_{5/2}}
\frac{
    2
    \sqrt{5}
    \sinh \left(\frac{\sqrt{5}}{2 \kappa} \right)
    }
    {
    (2\pi)^{-4}
    \kappa
    }
\del{\frac{5}{\kappa^2} + 4\pi^2 n^2 }^{-3}
,
\>
where $u = \sqrt{5} \frac{\abs{x-x'}-1/2}{\kappa}$ and
\<
a_{2, 0}
&=
\frac{\pi^4 \kappa^2}{50} 
\del{
    -5 + 12 \kappa^2
    + 6 \sqrt{5} \kappa \coth \del{\frac{\sqrt{5}}{2 \kappa}}
    +10 \coth \del{\frac{\sqrt{5}}{2 \kappa}}^2
}
,
\\
a_{2, 1}
&=
-\frac{2 \pi^4 \kappa^3}{25}
\del{
    3 \kappa
    + \sqrt{5} \coth \del{\frac{\sqrt{5}}{2 \kappa}}
}
,
\qquad
a_{2, 2}
=
\frac{2 \pi^4 \kappa^4}{25}
.
\>

\end{example}

Finally, we discuss Fourier feature approximations. 
The main ingredient of these approximations is the formula
\[
k_{(\cdot)}(x, x')
=
\sum\limits_{n \in \Z}
    \rho_{(\cdot)}(n)
    e^{2 \pi i n (x-x')}
=
\rho_{(\cdot)}(0)
+
2
\sum\limits_{n \in \N}
    \rho_{(\cdot)}(n)
    \cos (2 \pi n (x-x'))
.
\]
The sum on the right hand side can be approximated either deterministically by truncating the series, or randomly with Monte Carlo techniques. 
This corresponds respectively to the following two approximations of the process
\<
f^D_{(\cdot)}(x)
&=
\sum_{n=-N}^{N}
\sqrt{\rho_{(\cdot)}(n)}
\del[1]{
    w_{n, 1} \cos(2 \pi n x)
    +
    w_{n, 2} \sin(2 \pi n x)
}
,
&
w_{n,j} &\~[N](0,1)
,
\>
and
\<
\hspace{-0.5cm}
f^R_{(\cdot)}(x)
&=
\frac{\sigma}{\sqrt{N}}
\sum_{k=0}^{N-1}
\del[1]{
    w_{n, 1} \cos(2 \pi n_k x)
    +
    w_{n, 2} \sin(2 \pi n_k x)
}
,
&
n_k &\~\frac{\rho_{(\cdot)}(n)}{\sigma^2}
,
\>
where $w_{n,j}$ is defined identically.
Note that the kernel discussed here is defined via periodic summation defined in Section \ref{sec:torus}.

\subsection*{Sphere}

The presentation here is based on \textcite[Section 7.3]{devito19}. 
Assume $d>1$ and take $M = \bb{S}^d$, where $\bb{S}^d$ is $d$-dimensional sphere $\bb{S}^d \subseteq \R^{d+1}$. 
The 1-dimensional case discussed in the previous section can be handled similarly but requires some additional care. 

The eigenvalues of $\Delta_{\bb{S}^d}$ are $\lambda_n = n (n + d - 1), n \in \Z_+$. The eigenspace $\c{H}_n$ corresponding to $\lambda_n$ has dimension $d_n = (2 n + d - 1)\frac{\Gamma(n + d - 1)}{\Gamma(d)\Gamma(n+1)}$ and consists of spherical harmonics of degree $n$. The addition formula for spherical harmonics yields that for any orthonormal basis $f_{n, k}$ of eigenspace~$\c{H}_n$
\[
\sum_{k=1}^{d_n} f_{n, k}(x) f_{n, k}(x') = c_{n, d} \c{C}_n^{(d-1)/2} (\cos(d_M(x, x')))
\]
where $\c{C}_n^{(d-1)/2}$ are Gegenbauer polynomials and the constant $c_{n, d}$ is defined by
\[
c_{n, d} = \frac{d_n \Gamma((d+1)/2)}{2 \pi^{(d+1)/2} \c{C}_n^{(d-1)/2} (1)}
.
\]
From this we deduce that the formula for Mat\'{e}rn kernel on $\bb{S}^d$ is given by
\<
k_\nu(x, x')
&=
\frac{\sigma^2}{C_{\nu}}
\sum\limits_{n=0}^\infty \del{\frac{2 \nu}{\kappa^2} + \lambda_n}^{-\del{\nu+\frac{d}{2}}} \del{\sum\limits_{k=1}^{d_n} f_{n, k}(x) f_{n, k}(x')}
\\
&=
\frac{\sigma^2}{C_{\nu}}
\sum\limits_{n=0}^\infty \del{\frac{2 \nu}{\kappa^2} + n (n + d - 1)}^{-\del{\nu+\frac{d}{2}}} c_{n, d} \c{C}_n^{(d-1)/2} (\cos(d_M(x, x'))).
\>
Analogously for squared exponential kernel on $\bb{S}^d$, we obtain
\<
k_\infty(x, x')
&=
\frac{\sigma^2}{C_{\infty}}
\sum\limits_{n=0}^\infty e^{-\frac{\kappa^2}{2} \lambda_n} \del{\sum\limits_{k=1}^{d_n} f_{n, k}(x) f_{n, k}(x')}
\\
&=
\frac{\sigma^2}{C_{\infty}}
\sum\limits_{n=0}^\infty e^{-\frac{\kappa^2}{2} n (n + d - 1)} c_{n, d} \c{C}_n^{(d-1)/2} (\cos(d_M(x, x'))).
\>

Summarizing the above, for the sphere $\bb{S}^d$ we obtain the following.

\begin{example}[Mat\'{e}rn and squared exponential kernels on $\bb{S}^d$]
The Mat\'{e}rn and squared exponential kernels and the corresponding spectral densities are given as follows
\<
k_{\nu}(x, x')
&=
\frac{\sigma^2}{C_{\nu}}
\sum\limits_{n=0}^\infty \del{\frac{2 \nu}{\kappa^2} + n (n + d - 1)}^{-\del{\nu+\frac{d}{2}}} c_{n, d} \c{C}_n^{(d-1)/2} (\cos(d_M(x, x')))
,
\\
k_{\infty}(x, x')
&=
\frac{\sigma^2}{C_{\infty}}
\sum\limits_{n=0}^\infty e^{-\frac{\kappa^2}{2} n (n + d - 1)} c_{n, d} \c{C}_n^{(d-1)/2} (\cos(d_M(x, x')))
,
\\
\rho_\nu(n)
&=
\frac{\sigma^2}{C_{\nu}}
\del{\frac{2 \nu}{\kappa^2} + n (n + d - 1)}^{-\del{\nu+\frac{d}{2}}}
,
\\
\rho_\infty(n)
&=
\frac{\sigma^2}{C_{\infty}}
e^{-\frac{\kappa^2}{2} n (n + d - 1)}
,
\>
where $d_M(x, x')$ is the geodesic distance between $x, x' \in \bb{S}^d$, $\c{C}_n^{(d-1)/2}$ are Gegenbauer polynomials and
\<
c_{n, d}
&=
\frac{d_n \Gamma((d+1)/2)}{2 \pi^{(d+1)/2} \c{C}_n^{(d-1)/2} (1)}
&
&\t{with}
&
d_n
&=
(2 n + d - 1)\frac{\Gamma(n + d - 1)}{\Gamma(d)\Gamma(n+1)}
.
\>
Note that for every $n \in \Z_+$ there are $d_n$ Laplace--Beltrami eigenfunctions.
Thus, in the following Fourier feature approximation, we cannot apply the combinatorial simplification that yields the Gegenbauer polynomials, and instead work with spherical harmonics directly.
The generalized Fourier feature approximations, both deterministic and random, are given by
\<
f^D_{(\cdot)}(x)
&=
\sum\limits_{n=0}^{N-1}
    \sqrt{\rho_{(\cdot)}(n)}
    \sum\limits_{j=1}^{d_n}
        w_{n, j}
        f_{n, j}(x)
, & w_{n,j} &\~[N](0,1)
,
\>
and
\<
f^R_{(\cdot)}(x)
&=
\frac{\sigma}{\sqrt{N}}
\sum\limits_{k=0}^{N-1}
     \sum\limits_{j=1}^{d_{n_k}}
     w_{n_k, j}
     f_{n_k, j}(x)
,
&
n_k &\~\frac{\rho_{(\cdot)}(n)}{\sigma^2}
,
&
w_{n_k,j} &\~[N](0,1)
,
\>
where $f_{n, k}$ are the actual spherical harmonics forming the orthonormal basis of eigenspace $\c{H}_n$.
\end{example}

No closed form expressions for $k_\nu$ and $k_\infty$ are known to the authors. 
Nevertheless, approximating the series defining $k_\nu$ and $k_\infty$ by truncation gives a practical approach with reasonable error control.
Note that the larger $\nu$ is, the faster these series converge, and the more accurate the resulting approximations are.

\section{Proof of Proposition \ref{prop:matern-and-se-torus}} \label{apdx:torus-proof}

{
\def\thetheorem{\ref{prop:matern-and-se-torus}}
\begin{proposition}
The Mat\'{e}rn (squared exponential) kernel $k$ in  \eqref{eqn:matern-torus}  (resp. \eqref{eqn:rbf-torus}) is the covariance kernel of the Mat\'{e}rn (resp. squared exponential) Gaussian process in the sense of \textcite{whittle63}.
\end{proposition}
}

\begin{proof}
Following Section \ref{sec:comp_man}, the Mat\'{e}rn and square exponential kernels on a compact Riemannian manifold in the sense of \textcite{whittle63} are given by \eqref{eqn:mani-matern-formula} and \eqref{eqn:mani-rbf-formula}.
For the sake of this proof we denote these kernels by $k^{(w)}_{(\cdot)}$ and the kernels defined by periodic summation (equations \eqref{eqn:matern-torus}, \eqref{eqn:rbf-torus}) by $k^{(p)}_{(\cdot)}$. 
We prove here that $k^{(p)}_{(\cdot)}$ are equal to $k^{(w)}_{(\cdot)}$.

To make equations \eqref{eqn:mani-matern-formula} and \eqref{eqn:mani-rbf-formula} explicit for~$\bb{T}^d = \R^d / \Z^d$, we need to compute the eigenfunctions and eigenvalues of Laplace--Beltrami operator~$\Delta_g$ on $\bb{T}^d$.
This is not difficult, since $\bb{T}^d$ is equipped with the quotient metric, which is flat.
In particular, this amounts to considering the eigenfunctions of Euclidean Laplacian, which are sines and cosines (complex exponentials), and leaving only those which are $1$-periodic.
The procedure is described in detail in \textcite{gordon2000}, and yields the following.
For $\tau \in \Z_+^d, \tau \neq 0$ the pair of functions $f_{\tau,1}(x) = \sqrt{2} \cos \del{2 \pi \tau \cdot x}$ and $f_{\tau,2}(x) = \sqrt{2} \sin \del{2 \pi \tau \cdot x}$ are eigenfunctions of the Laplace--Beltrami operator, corresponding to the eigenvalue $\lambda_\tau = 4 \pi^2 |\tau|^2$.
It is important to note that $\tau$ and $-\tau$ correspond to the same (up to a sign change) pair of eigenfunctions.
Together with the function $f_0(x) = 1$ corresponding to the eigenvalue $\lambda_0 = 0$, they form the orthonormal basis of $L^2(\bb{T}^d)$. 
To unify notation, we write $f_{0,1}(x) = 1$ and $f_{0,2}(x) = 0$. 

Since the series defining $k^{(w)}_{(\nu)}$ is unconditionally convergent \cite{devito19}, we obtain
\[
k^{(w)}_{\nu}(x, x')
=
\frac{\sigma^2}{C_{\nu}}
\sum\limits_{
    \tau \in \c{I}
}
    \del{
        \frac{2 \nu}{\kappa^2}
        + 4 \pi^2 \abs{\tau}^2
    }^{-\nu - \frac{d}{2}}
    \del{
    f_{\tau,1}(x) f_{\tau,1}(x')
    +
    f_{\tau,2}(x) f_{\tau,2}(x')
    }
,
\]
where $\c{I} \subseteq \Z^d$ is a maximal subset of $\Z^d$ such that $\tau \in \c{I}$ and $\tau \not = 0$ implies $-\tau \not\in \c{I}$.
This, using identities $\cos(x-y) = \cos(x) \cos(y) + \sin(x) \sin(y)$ and $\cos(x) = \del{\cos(x) + \cos(-x)}/2$,~becomes
\[
k^{(w)}_{\nu}(x, x')
=
\frac{\sigma^2}{C_{\nu}}
\sum\limits_{
    \tau \in \Z^d
}
    \del{
        \frac{2 \nu}{\kappa^2}
        + 4 \pi^2 \abs{\tau}^2
    }^{-\nu - \frac{d}{2}}
    \cos \del{2 \pi \tau \cdot (x-x')}
.
\]
At the same time, the generalized Poisson summation formula gives
\[
\hspace{-0.05cm}
k^{(p)}_{\nu}(x, x')
=
\frac{\sigma^2}{C'_{\nu}}
\sum\limits_{
    n \in \Z^d
}
    S(n)
    e^{2 \pi i n \cdot (x-x')}
=
\frac{\sigma^2}{C'_{\nu}}
\del[2]{
    \sum\limits_{
        \tau \in \Z^d
    }
        S(\tau)
        \cos \del{2 \pi \tau \cdot (x-x')}
}
,
\]
where $S$ is the spectral density of Mat\'{e}rn kernel on $\bb{R}^d$. 
This is given by \cite[Section 4.2.1]{rasmussen06}
\[
S(\xi)
=
\frac{
    2^d
    \pi^{\frac{d}2}
    \Gamma(\nu+\frac{d}2)
    (2\nu)^\nu
}{
    \Gamma(\nu)\kappa^{2\nu}
}
\left(
    \frac{2\nu}{\kappa^2}
    +
    4\pi^2 \abs{\xi}^2
\right)^{-\left(\nu+\frac{d}2\right)}
.
\]
Thus for finite $\nu$ we have
\[ \label{eqn:mult-add-consts-matern}
k^{(p)}_{\nu}(x, x')
=
\frac{
    C_{\nu}
    2^d
    \pi^{\frac{d}2}
    \Gamma(\nu+\frac{d}2)
    (2\nu)^\nu
}{
    C'_{\nu}
    \Gamma(\nu)\kappa^{2\nu}
}
k^{(w)}_{\nu}(x, x')
.
\]
Recalling that $C_\nu$ and $C'_{\nu}$ are chosen so that $k^{(p)}_{\nu}(x, x) = \sigma^2 = k^{(w)}_{\nu}(x, x)$, we see that $k^{(p)}_{\nu}(x, x') = k^{(w)}_{\nu}(x, x')$, which gives the claim.

The argument for squared exponential kernel ($\nu = \infty$) is essentially the same. 
In this case we have
\<
k^{(w)}_{\infty}(x, x')
&=
\frac{\sigma^2}{C_{\infty}}
\sum\limits_{
    \tau \in \Z^d
}
    \exp \del{
        -
        2 \pi^2 \kappa^2 \abs{\tau}^2
    }
    \cos \del{2 \pi \tau \cdot (x-x')}
,
\\
k^{(p)}_{\infty}(x, x')
&=
\frac{\sigma^2}{C'_{\infty}}
\del[2]{
    \sum\limits_{
        \tau \in \Z^d
    }
        S(\tau)
        \cos \del{2 \pi \tau \cdot (x-x')}
}
,
\>
but this time with 
\[
S(\xi)
=
\sigma^2
(2 \pi \kappa^2)^{d/2}
e^{-2 \pi^2 \kappa^2 \abs{\xi}^2}
.
\]
This gives
\[ \label{eqn:mult-add-consts-se}
k^{(p)}_{\infty}(x, x')
=
\frac{
    C_{\infty}
    (2 \pi \kappa^2)^{d/2}
}{
    C'_{\infty}
}
k^{(w)}_{\infty}(x, x'),
\]
which translates into $k^{(p)}_{\infty}(x, x') = k^{(w)}_{\infty}(x, x')$ with our specific choice of constants $C_{\infty}$ and $C'_{\infty}$,
and thus completes the proof.
\end{proof}

\section{Theory: compact Riemannian manifolds without boundary} \label{apdx:theory}

Here we introduce an appropriate formalism for the stochastic partial differential equations \eqref{eqn:spde-matern} and \eqref{eqn:spde-rbf} and prove that their solutions are the reproducing kernels of the Sobolev and diffusion spaces given by \textcite{devito19}.

Let $(M,g)$ be a compact connected Riemannian manifold without boundary,\footnote{Such a manifold is automatically complete, since a compact metric space is always complete.} and let $\lap_g$ be the Laplace--Beltrami operator defined on the space $C^\infty(M)$ of smooth functions on $M$.
Let $L^2(M)$ denote the space of (almost everywhere equal equivalence classes of) functions on $M$ which are square integrable with respect to the Riemannian volume measure.

\begin{theorem}
  The operator $- \lap_g: C^\infty(M) \to L^2(M)$ uniquely extends to a self-adjoint unbounded operator from some domain $D(\lap_g) \subseteq L^2(M)$ to $L^2(M)$, and this extension, denoted again by $- \lap_g$, is a positive operator.
\end{theorem}

\begin{proof}
\textcite[Theorem 2.4]{strichartz1983}.
\end{proof}

This allows one to apply the spectral theorem for self-adjoint unbounded operators, which, loosely speaking, diagonalizes such operators and enables us to introduce a functional calculus for them.
The general statement of the spectral theorem for unbounded self-adjoint operators can be found in various textbooks---see, for instance, \textcite[Chapters XIX and XX]{lang1993} or \textcite[Chapter VIII]{reed1980}.
For our setting, we do not need this general statement, as there is a separate theorem for the special case of the Laplace--Beltrami operator on a compact manifold, commonly referred to as the Sturm-Liouville decomposition.

\begin{theorem}[Sturm--Liouville decomposition]
Let $(M,g)$ be a compact Riemannian manifold without boundary.
Then there exists an orthonormal basis $\{f_n\}_{n\in\Z_+}$, of $L^2(M)$ such that $-\lap_g f_n = \lambda_n f_n$ with $0 = \lambda_0 < \lambda_1 \leq .. \leq \lambda_n$ and $\lambda_n\-> \infty$ as $n\->\infty$.
Moreover, $-\lap_g$ admits the representation
\[
-\lap_g f = \sum_{n=0}^\infty \lambda_n \innerprod{f}{f_n} f_n,
\]
which converges unconditionally in $L^2(M)$ for all $f \in D(\lap_g)$.
\end{theorem}
\begin{proof}
See \textcite[page 139]{chavel1984} or \textcite[Theorem 44]{canzani13}.
\end{proof}

This allows one to define a (possibly unbounded) operator $\Phi(-\lap_g)$ for any Borel measurable function $\Phi: [0, +\infty) \to \R$ by
\[ \label{eqn:func_calc}
\Phi(-\lap_g) f = \sum\limits_{n=0}^\infty \Phi(\lambda_n) \innerprod{f}{f_n} f_n
\]
with domain given by
\[
D(\Phi(-\lap_g)) = \Set*{f \in L^2(M)}{\sum\limits_{n=0}^\infty \abs{\Phi(\lambda_n)}^2 \abs{\innerprod{f}{f_n}}^2 < \infty}
.
\]
This idea is called the functional calculus for the operator $-\lap_g$.
It allows us to formally define operators from the SPDEs under consideration with
\<
\label{eqn:matern_op_def}
\del{\frac{2 \nu}{\kappa^2} - \lap_g}^{\frac{\nu}{2}+\frac{d}{4}} f
&=
\sum_{n=0}^\infty \del{\frac{2 \nu}{\kappa^2} + \lambda_n}^{\frac{\nu}{2}+\frac{d}{4}} \innerprod{f}{f_n} f_n
,
&
&\t{using} \Phi(\lambda) = \del{\frac{2 \nu}{\kappa^2} + \lambda}^{\frac{\nu}{2}+\frac{d}{4}},
\\
\label{eqn:se_op_def}
e^{-\frac{\kappa^2}{4} \Delta} f
&=
\sum_{n=0}^\infty e^{\frac{\kappa^2 \lambda_n}{4}} \innerprod{f}{f_n} f_n
,
&
&\t{using} \Phi(\lambda) = e^{\frac{\kappa^2 \lambda}{4}}
.
\>
Denote these operators by $\c{L}$.
We now proceed to define an appropriate formalism for the SPDEs
\[ \label{eqn:the_spde}
\c{L}f = \c{W}
.
\]
We start by introducing a notion of generalized Gaussian random fields.
\begin{definition}[Definition 3.2.10 of \textcite{lototsky17}] \label{dfn:lot3.2.10}
  A zero-mean generalized Gaussian field $\f{F}$ over a Hilbert space $H$ is a collection of Gaussian random variables $\cbr{\f{F}(h)}_{h \in H}$ with the properties
  \1
    $\E \del{\f{F} (h)} = 0$ ~for all~ $h \in H$,
  \2
    There exists a bounded, linear, self-adjoint, non-negative operator $K$ on $H$ (called the covariance operator of $\f{F}$) such that
    \[
    \E \del{\f{F}(h) \f{F}(g)} = \innerprod{K h}{g}_H
    \]
    for all $h, g \in H$.
  \0
\end{definition}

A zero-mean generalized Gaussian field $\c{W}$ over a Hilbert space $H$ with identity $I: H \to H$ serving as covariance operator is called the \emph{standard Gaussian white noise} over $H$. 

Let $\c{W}$ be said white noise over $L^2(M)$.
Up to a normalizing constant which ensures that the solution has the right variance, this is equal to the right hand side of equation \eqref{eqn:the_spde}.
We do not dwell on this constant until the very end of this section, where it appears naturally as the normalizing constant of the resulting kernel.

It is easy to see that the generalized Gaussian field which we have just defined can be thought of as an operator from $H$ to the space~$L^2(\Omega)$ of zero mean random variables with finite variance.
From this view, the Gaussian white noise $\c{W}$ is an isometric embedding.

To give more intuition, we explicitly consider how the usual concept of a Gaussian process embeds into this generalization.
Let $f \~[GP](0, k(x, x'))$ be a Gaussian process over a manifold $M$ with covariance function $k(x, x')$.
Assume that $k$ is regular enough to consider samples of $f$ as elements of $L^2(M)$.
Almost every practically reasonable covariance function will be regular enough in this sense, so this assumption is not restrictive.
The generalized Gaussian field over $L^2(M)$ corresponding to $f$ will be the operator $\f{F}_f (g) = \innerprod{f}{g}_{L^2(M)}$ for which
\<
&\E \del{\f{F}_f (h) \f{F}_f (g)}
=
\E \del{ \innerprod{f}{h}_{L^2(M)}  \innerprod{f}{g}_{L^2(M)}}
=
\E \int\limits_M \int\limits_M f(x) h(x) f(y) g(y) \d x \d y
\\
&=
\int\limits_M \int\limits_M \E \del{f(x) f(y)} h(x) g(y) \d x \d y
=
\int\limits_M \int\limits_M k(x, y) h(x) g(y) \d x \d y
=
\innerprod{K h}{g}_{L^2(M)}
,
\>
where $K: L^2(M) \to L^2(M)$ is an operator defined by $(K h) (x) = \int_M k(x, y) h(y) dy$. 
Note that $\c{W}$ is much less regular and \emph{cannot} be represented this way.

Now we are ready to introduce the formal meaning of the SPDEs.

\begin{definition}[Definition 4.2.1 of \textcite{lototsky17}] \label{dfn:lot4.2.1}
  Let $H$ be a Hilbert space and let $\c{L}: H \to L^2(M)$ be a bounded linear operator.
  The zero-mean generalized Gaussian random field~$\f{F}$ over $H$ is a solution of the equation
  \[ \label{eqn:formal_spde}
  \c{L} \f{F} = \c{W}
  \]
  if for every $g \in L^2(M)$
  \[
    \f{F}(\c{L}^* g) = \c{W}(g).
  \]
\end{definition}
\begin{theorem}[Theorem 4.2.2 of \textcite{lototsky17}] \label{thm:lot4.2.2}
  If $\c{L}$ from definition \ref{dfn:lot4.2.1} is invertible, then a zero-mean generalized Gaussian field $\f{F}$ over $H$ defined by
  \[ \label{eqn:spde_solution}
    \f{F}(h) = \c{W} \del{\del{\c{L}^{-1}}^* h}
  \]
  is the unique solution of the equation \eqref{eqn:formal_spde}.
\end{theorem}
Informally, this means that $\f{F} = \c{L}^{-1} W$ is the solution of $\c{L} \f{F} = W$.
The operator $\c{L}^{-1} I \c{L}^{-1} = \c{L}^{-2}$ is the covariance operator of $\f{F}$, which is an integral operator with some kernel $k$, which in its turn is the covariance function of $\f{F}$ when viewed as an ordinary Gaussian process over the manifold $M$. 
The kernel $k$ is easily derived from formulas \eqref{eqn:matern_op_def} and \eqref{eqn:se_op_def}---in the following, we will rigorously arrive at this result.

First, we need to introduce appropriate spaces $H$ to make $\c{L}: H \to L^2(M)$ into a bounded linear bijection.

To better fit our presentation into the existing mathematical framework, we would like the operator~\eqref{eqn:matern_op_def} to have $2 \nu / \kappa^2 = 1$. 
The next statement shows that this assumption does not lead to any loss of generality.

\begin{restatable}{proposition}{reductiontheorem}
\label{thm:spec_case_reduction}
Consider a manifold $(M, \tilde{g})$ with $\tilde{g} = \frac{2 \nu}{\kappa^2} g$, then for $\f{F}$ and $\f{G}$ satisfying 
\<
  &\del{\frac{2 \nu}{\kappa^2} - \lap_g}^{\frac{\nu}{2}+\frac{d}{4}} \f{F} = \c{W}
  ,
  &
  &\del{1 - \lap_{\tilde{g}}}^{\frac{\nu}{2}+\frac{d}{4}} \f{G} = \c{W}_{\tilde{g}},
\>
it is true that $\f{F} = \del{\frac{\kappa^2}{2 \nu}}^{\frac{\nu + d}{2}} \f{G}$.
\end{restatable}

We postpone the proof until after we have introduced the remaining formalism.
For the time being, we assume $2 \nu / \kappa^2 = 1$ when dealing with operator~\eqref{eqn:matern_op_def}.

We proceed to define Sobolev spaces on $M$ which will serve as an appropriate $H$ for the operator~\eqref{eqn:matern_op_def}. 
\begin{definition} \label{dfn:sobolev_space}
  Consider $s \in (0, +\infty)$.
  Define the operator $(1-\lap_g)^{-\frac{s}{2}}$ via \eqref{eqn:func_calc}.
  We say that a distribution $f \in \c{D}'(M)$ belongs to the Sobolev space $H^s(M)$ if and only if there exists $g \in L^2(M)$ such that $f = (1-\lap_g)^{-\frac{s}{2}} g$.
  We define the norm with~$\norm{f}_{H^s} = \norm{g}_{L^2(M)}$, and the inner product with~$\innerprod{f}{h}_{H^s(M)} = \innerprod{g}{u}_{L^2(M)}$, if $h = (1-\lap_g)^{-\frac{s}{2}} u \in H^s(M)$.
\end{definition}

This is one of several equivalent definition of Sobolev spaces on Riemannian manifolds, other definitions could be found in \textcite[Theorem 3]{devito19} along with a proof of their equivalence. 
It can be seen, thanks to our assumption $2 \nu / \kappa^2 = 1$, that these spaces are particularly suitable domains for this operator \ref{eqn:matern_op_def}, because these spaces are image of the inverse operator acting on $L^2(M)$.

In addition, following \textcite{devito19}, we introduce \emph{diffusion spaces}, which will be suitable for~\eqref{eqn:se_op_def}.
\begin{definition} \label{dfn:diffusion_space}
  Consider $t \in (0, +\infty)$. Define operator $e^{\frac{t}{2} \lap_g}$ via \eqref{eqn:func_calc}. We say that a distribution $f \in \c{D}'(M)$ belongs to the diffusion space $\c{H}^t(M)$ if and only if there exists $g \in L^2(M)$ such that $f = e^{\frac{t}{2} \lap_g} g$. We define the norm with~$\norm{f}_{\c{H}^t} = \norm{g}_{L^2(M)}$ and the inner product with~$\innerprod{f}{h}_{\c{H}^t(M)} = \innerprod{g}{u}_{L^2(M)}$, if $h = e^{\frac{t}{2} \lap_g} u \in \c{H}^t(M)$.
\end{definition}

Both of these types of spaces are Hilbert spaces \cite{devito19}. 
This gives the following.

\begin{theorem}
  The operators
  \< \label{eqn:operators_with_spaces}
  \del{1 - \lap_g}^{\frac{\nu}{2}+\frac{d}{4}}&: H^{\nu + \frac{d}{2}} \to L^2(M)
  &
  e^{-\frac{\kappa^2}{4} \Delta}&: \c{H}^{\frac{\kappa^2}{2}} \to L^2(M)
  \>
  are bounded and invertible.
\end{theorem}
\begin{proof}
Immediate by definition of $H^{\nu + \frac{d}{2}}$ and $\c{H}^{\frac{\kappa^2}{2}}$.
\end{proof}

Now, we suppose that $\c{L}$ is one of the operators from \eqref{eqn:operators_with_spaces} and $H$ is the corresponding space such that $\c{L}: H \to L^2(M)$. 
Since the conditions of Theorem \ref{thm:lot4.2.2} are satisfied, the solution of \eqref{eqn:formal_spde} is a zero-mean generalized Gaussian field $\f{F}$ defined by \eqref{eqn:spde_solution}. 
We now compute the covariance operator of~$\f{F}$, which is
\[
\E \del{\f{F}(h) \f{F}(g)} = \E \del{
\c{W}\del{\del{\c{L}^{-1}}^* h}
\c{W}\del{\del{\c{L}^{-1}}^* g}
} = \innerprod{\del{\c{L}^{-1}}^* h}{\del{\c{L}^{-1}}^* g}_{L^2(M)},
\]
and since $\innerprod{a}{b}_{H} = \innerprod{\c{L} a}{\c{L} b}_{L^2(M)}$ is clear from definitions \ref{dfn:sobolev_space} and \ref{dfn:diffusion_space}, we have for every $h \in H$ and $u \in L^2(M)$
\[
\innerprod{\del{\c{L}^{-1}}^* h}{u}_{L^2(M)}
= 
\innerprod{h}{\c{L}^{-1} u}_{H}
=
\innerprod{\c{L} h}{\c{L} \c{L}^{-1} u}_{L^2(M)}
=
\innerprod{\c{L} h}{u}_{L^2(M)}
.
\]
This means that $\del{\c{L}^{-1}}^* = \c{L}$ and thus
\[
\E \del{\f{F}(h) \f{F}(g)}
=
\innerprod{\del{\c{L}^{-1}}^* h}{\del{\c{L}^{-1}}^* g}_{L^2(M)}
=
\innerprod{\c{L} h}{\c{L} g}_{L^2(M)}
=
\innerprod{h}{g}_{H},
\]
so $\f{F}$ is a Gaussian white noise over $H$.

We now want to obtain a Gaussian process indexed by $M$ from the generalized Gaussian field $\f{F}$. 
That is, we want to define $\f{F}(x)$ for $x \in M$ and to compute covariance function of such $\f{F}$. 
This can be easily done thanks to the fact that $H$ is a reproducing kernel Hilbert space, which was proven in \textcite[Theorem 8, Proposition 2]{devito19}---note that for the Sobolev spaces $H^s$ under consideration we always have $s>d/2$ since $s = \nu + d/2$, $\nu > 0$.

Let $k(x, x')$ be the reproducing kernel of $H$. 
It is natural to define $\f{F}(x) = \f{F}(k(x, \cdot))$ for $x \in M$. This $\f{F}(x)$ will be a Gaussian random variable by Definition \ref{dfn:lot3.2.10}. 
Moreover,
\[
\E \del{\f{F}(x) \f{F}(x')} = \innerprod{k(x, \cdot)}{k(x', \cdot)}_{H} = k(x, x')
\]
by the definition of a reproducing kernel. It follows that $\cbr{\f{F}(x)}_{x \in M}$ is a Gaussian process in the standard sense with zero mean and covariance function $k$ which is the reproducing kernel of $H$.\footnote{It is easy to see that $\c{B}_{\f{F}}(g) := \innerprod{\f{F}}{g}_H$, where $\f{F}$ is the Gaussian process on $M$, is the generalized Gaussian field $\f{F}$ we started with.}

The reproducing kernels for Sobolev spaces are given in \textcite[Proposition 2]{devito19} as
\[
k(x, x')
=
\sum_{n=0}^\infty \del{1 + \lambda_n}^{-\nu-\frac{d}{2}} f_n(x)f_n(x').
\]
An analogous statement is true for the Diffusion spaces, giving
\[
k(x, x')
=
\sum_{n=0}^\infty e^{-\frac{\kappa^2}{2} \lambda_n} f_n(x)f_n(x')
\]
with the proof repeating the proof of \cite[Proposition 2]{devito19} mutatis mutandis.

Thus, the kernels normalized to have average variance $\sigma^2$ are given by
\<
k_{\nu}(x, x')
&=
\frac{\sigma^2}{C_{\nu}}
\sum_{n=0}^\infty \del{1 + \lambda_n}^{-\nu-\frac{d}{2}} f_n(x)f_n(x')
\\
\label{eqn:k_inf_apdx_orig}
k_{\infty}(x, x')
&=
\frac{\sigma^2}{C_{\infty}}
\sum_{n=0}^\infty e^{-\frac{\kappa^2}{2} \lambda_n} f_n(x)f_n(x'),
\>
where the constant $C_{(\cdot)}$ is chosen so that $\vol_g(M)^{-1} \int k_{(\cdot)}(x, x) \d x = \sigma^2$. 
In some cases, for instance when $M$ is a homogeneous manifold, $k_{(\cdot)}(x, x)$ will not depend on $x$, so $k(x, x) = \sigma^2$ can be satisfied.\footnote{It is not known to the authors if homogeneous manifolds are the only manifolds for which $k(x, x)$ does not depend on $x$. It seems like an interesting mathematical problem to describe manifolds with this property. It is even more interesting to describe how the way $k(x, x)$ changes depending on $x$ is determined by the geometry of $M$.}

Note that throughout the above, we still assumed $\kappa$ is chosen such that $2 \nu / \kappa^2 = 1$. 
To show this assumption was indeed taken without loss of generality, we prove the following.

\reductiontheorem*
\begin{proof}
  First, let us verify that the equation to the left is well-defined.
  To do this, we must check that operator \eqref{eqn:matern_op_def} is bounded and invertible for general $\kappa, \nu > 0$.
  Fix $f \in H^{\nu + \frac{d}{2}}$ and find $g \in L^2(M)$ such that $f = (1 - \lap_g)^{-\frac{\nu}{2}-\frac{d}{4}} g$.
  Write $g = \sum_{n=0}^\infty \alpha_n f_n$ using the basis $\cbr{f_n}$ consisting of Laplacian eigenfunctions, so $f = \sum_{n=0}^\infty \del{1+\lambda_n}^{-\frac{\nu}{2}-\frac{d}{4}} \alpha_n f_n$.
  Noting that 
  \<
  \min\del{\frac{2\nu}{\kappa^2}, 1} \leq \frac{\frac{2\nu}{\kappa^2} + \lambda_n}{1+\lambda_n} \leq \max\del{1, \frac{2\nu}{\kappa^2}}
  \>
  we can write
  \<
  &\norm{
    \del{\frac{2 \nu}{\kappa^2} - \lap_g}^{\frac{\nu}{2}+\frac{d}{4}} f
  }_{L^2(M)}^2
  =
  \norm{
    \sum\limits_{n=0}^\infty
    \del{\frac{2 \nu / \kappa^2 + \lambda_n}{1+\lambda_n}}^{\frac{\nu}{2}+\frac{d}{4}}
    \alpha_n f_n
  }_{L^2(M)}^2
  \\
  &=
  \sum\limits_{n=0}^\infty
  \del{\frac{2 \nu / \kappa^2 + \lambda_n}{1+\lambda_n}}^{\nu+\frac{d}{2}}
  \alpha_n^2
  \leq
  \sum\limits_{n=0}^\infty
  \max\del{1, \frac{2\nu}{\kappa^2}}^{\nu+\frac{d}{2}}
  \alpha_n^2
  \\
  &=
  \max\del{1, \frac{2\nu}{\kappa^2}}^{\nu+\frac{d}{2}} \norm{g}_{L^2(M)}^2
  =
  \max\del{1, \frac{2\nu}{\kappa^2}}^{\nu+\frac{d}{2}} \norm{f}_{H^{\nu + \frac{d}{2}}}^2
  ,
  \>
  which proves boundedness as well as $f \in D \del{\del{\frac{2 \nu}{\kappa^2} - \lap_g}^{\frac{\nu}{2}+\frac{d}{4}}}$.
  To prove the operator is invertible, write
  \<
  &\norm{
    \del{\frac{2 \nu}{\kappa^2} - \lap_g}^{\frac{\nu}{2}+\frac{d}{4}} f
  }_{L^2(M)}^2
  =
  \sum\limits_{n=0}^\infty
  \del{\frac{2 \nu / \kappa^2 + \lambda_n}{1+\lambda_n}}^{\nu+\frac{d}{2}}
  \alpha_n^2
  \\
  &\geq
    \min\del{\frac{2\nu}{\kappa^2},1}^{\nu+\frac{d}{2}}
    \sum\limits_{n=0}^\infty
    \alpha_n^2
  =
    \min\del{\frac{2\nu}{\kappa^2},1}^{\nu+\frac{d}{2}}
    \norm{f}_{H^{\nu + \frac{d}{2}}}^2
  .
  \>
  Now, consider how a change of the metric from $g$ to $\tilde{g} = \frac{2 \nu}{\kappa^2} g$ changes the objects under consideration. 
  This is given by the standard expressions
  \[
    \lap_{\tilde{g}} = \frac{\kappa^2}{2 \nu} \lap_g
    ,
    \qquad
    \widetilde{dx} = \del{\frac{2 \nu}{\kappa^2}}^{d/2} dx
    ,
  \]
  which in turn gives
  \[
    \tilde{\lambda}_n = \frac{\kappa^2}{2 \nu} \lambda_n
    ,
    \quad
    \tilde{f_n} = \del{\frac{2 \nu}{\kappa^2}}^{-d/4} f_n
    ,
    \quad
    \innerprod{f}{g}_{\tilde{g}} = \del{\frac{2 \nu}{\kappa^2}}^{d/2} \innerprod{f}{g}
    ,
    \quad
    \c{W}_{\tilde{g}} = \del{\frac{2 \nu}{\kappa^2}}^{d/4} \c{W}
    .
  \]
  With this, we have
  \<
    \del{1 - \lap_{\tilde{g}}}^{\frac{\nu}{2}+\frac{d}{4}} \f{G}
    = 
    \sum_{n=0}^\infty \del{1 + \tilde{\lambda}_n}^{\frac{\nu}{2}+\frac{d}{4}} \innerprod{\f{G}}{\tilde{f_n}}_{\tilde{g}} \tilde{f_n}
    = 
    \sum_{n=0}^\infty \del{1 + \frac{\kappa^2}{2 \nu} \lambda_n}^{\frac{\nu}{2}+\frac{d}{4}} \innerprod{\f{G}}{f_n} f_n
    \\
    =
    \del{\frac{\kappa^2}{2 \nu}}^{\frac{\nu}{2}+\frac{d}{4}}
    \sum_{n=0}^\infty \del{\frac{2 \nu}{\kappa^2} + \lambda_n}^{\frac{\nu}{2}+\frac{d}{4}} \innerprod{\f{G}}{f_n} f_n
    =
    \del{\frac{\kappa^2}{2 \nu}}^{\frac{\nu}{2}+\frac{d}{4}}
    \del{\frac{2 \nu}{\kappa^2} - \lap_g}^{\frac{\nu}{2}+\frac{d}{4}}
    \f{G}
    .
  \>
  This means that $\f{G}$ is a solution of
  \[
    \del{\frac{\kappa^2}{2 \nu}}^{\frac{\nu}{2}+\frac{d}{4}}
    \del{\frac{2 \nu}{\kappa^2} - \lap_g}^{\frac{\nu}{2}+\frac{d}{4}}
    \f{G}
    =
    \del{\frac{2 \nu}{\kappa^2}}^{d/4} \c{W}
    .
  \]
  Gathering all constants, we get that $\f{F} = \del{\frac{\kappa^2}{2 \nu}}^{\frac{\nu}{2}+\frac{d}{4}} \del{\frac{2 \nu}{\kappa^2}}^{-d/4} \f{G} = \del{\frac{\kappa^2}{2 \nu}}^{\frac{\nu+d}{2}} \f{G}$ is the solution to
  \[
  \del{\frac{2 \nu}{\kappa^2} - \lap_g}^{\frac{\nu}{2}+\frac{d}{4}} \f{F} = \c{W}
  \]
  which proves the statement.
\end{proof}

This means that the kernel of a Gaussian process solving $\del{\frac{2 \nu}{\kappa^2} - \lap_g}^{\frac{\nu}{2}+\frac{d}{4}} \f{F} = \c{W}$ is proportional to
\[
k(x, x')
=
\sum_{n=0}^\infty \del{1 + \tilde{\lambda}_n}^{-\nu-\frac{d}{2}} \tilde{f_n}(x)\tilde{f_n}(x') 
=
\del{\frac{\kappa^2}{2 \nu}}^{-\nu-d}
\sum_{n=0}^\infty \del{\frac{2 \nu}{\kappa^2} + \lambda_n}^{-\nu-\frac{d}{2}} f_n(x) f_n(x') 
.
\]
Re-normalizing this kernel, we finally get
\[
k_{\nu}(x, x')
=
\frac{\sigma^2}{C_{\nu}}
\sum_{n=0}^\infty \del{\frac{2 \nu}{\kappa^2} + \lambda_n}^{-\nu-\frac{d}{2}} f_n(x) f_n(x') 
,
\]
where $C_{\nu}$ is chosen as above and $\kappa$ can now be any positive number. Together with
\[
k_{\infty}(x, x')
=
\frac{\sigma^2}{C_{\infty}}
\sum_{n=0}^\infty e^{-\frac{\kappa^2}{2} \lambda_n} f_n(x)f_n(x')
\]
given in \eqref{eqn:k_inf_apdx_orig}, this gives the kernels we sought, and concludes our presentation.

\end{document}